\newcommand{\commentAGI}[1]{#1}  
\newcommand{\commentARXIV}[1]{}

\commentAGI{
\documentclass{article}
}
\commentARXIV{
\documentclass{llncs}  
}

\newcommand{\comment}[1]{} 
\newcommand{\commentvisible}[1]{} 

\commentAGI{
\textwidth = 6.5 in 
\textheight = 9 in 
\oddsidemargin = 0.0 in
\evensidemargin = 0.0 in
\topmargin = 0.0 in
\headheight = 0.0 in
\headsep = 0.0 in
}

\input{glyphtounicode}  
\usepackage[colorlinks,\commentAGI{pagebackref,}linktocpage]{hyperref}

\usepackage{algorithmicx} 
\usepackage{algpseudocode} 
\usepackage{algorithm}
\usepackage{caption}

\usepackage{amssymb}  
\usepackage{amsmath}   

\commentAGI{
\usepackage{amsthm} 
\newtheorem{theorem}{Theorem}

\newtheorem{definition}{Definition}
}

\usepackage{amsopn}

\DeclareMathOperator*{\argmin}{arg\,min}

\usepackage{color}
\definecolor{orange}{RGB}{255,127,0}
\definecolor{brown}{RGB}{150,70,0}
\definecolor{green}{RGB}{127,255,127}
\definecolor{darkgreen}{RGB}{0,127,0}
\definecolor{blue}{RGB}{127,127,255}
\definecolor{lightblue}{RGB}{150,150,255}
\definecolor{darkblue}{RGB}{0,0,127}
\definecolor{red}{RGB}{255,90,90}
\definecolor{grey}{RGB}{127,127,127}
\definecolor{pink}{RGB}{255,180,180}
\newcommand{\blue}[1]{\textcolor{blue}{{#1}}}

\newcommand{\red}[1]{\textcolor{red}{{#1}}}

\usepackage{xspace}   
\newcommand{\xaxis}{$x$-axis\xspace}
\newcommand{\yaxis}{$y$-axis\xspace}

\newcommand{\Steps}{\ensuremath{S}} 
\newcommand{\ESteps}{\ensuremath{\mathbb{S}}} 
\newcommand{\EMemory}{\ensuremath{\mathbb{M}}} 
\newcommand{\Length}{\ensuremath{L}} 
\newcommand{\ELength}{\ensuremath{\mathbb{L}}} 
\newcommand{\ELS}{\ensuremath{\mathbb{LS}}} 
\newcommand{\Response}{\ensuremath{R}} 
\newcommand{\EResponse}{\ensuremath{\mathbb{R}}} 
\newcommand{\Verification}{\ensuremath{W}}  
\newcommand{\EVerification}{\ensuremath{\mathbb{W}}}  
\newcommand{\Effort}{\ensuremath{F}}   
\newcommand{\EEffort}{\ensuremath{\mathbb{F}}}   
\newcommand{\Run}{\ensuremath{Run}}   
\newcommand{\EBids}{\ensuremath{\mathbb{B}}}   
\newcommand{\Hardness}{\ensuremath{\hbar}}   

\newcommand{\Acc}{\ensuremath{\mathbb{A}}\xspace}
\newcommand{\AccSet}{\ensuremath{{\cal{A}}}\xspace}
\newcommand{\Var}[1]{\ensuremath{\mbox{\sffamily{Var}}[#1]\xspace}}

\newcommand{\Psy}{\Psi}
\newcommand{\Psydiff}[1]{\ensuremath{\Psy_{#1}\xspace}}

\newcommand{\TAUC}{} 
\newcommand{\nuT}{} 
\newcommand{\nuE}{\nu} 

\newcommand{\natlangprog}[1]{`{\ttfamily #1}'}

\usepackage{graphicx}

\commentAGI{
\title{
Universal Psychometrics Tasks: difficulty, composition and decomposition}   
}
\commentARXIV{
\title{
Stochastic Tasks: Difficulty and Levin Search}   
}
\author
{
	Jos\'{e} Hern\'{a}ndez-Orallo\commentAGI{\\
	{\normalsize\em DSIC, Universitat Polit\`ecnica de Val\`encia, Spain}\\
	{\normalsize \tt jorallo@dsic.upv.es}
	}
}
\commentARXIV{
\institute{DSIC, Universitat Polit\`ecnica de Val\`encia, Spain\\  
\email{jorallo@dsic.upv.es}}
}

\commentAGI{
\date{\today}
}

\begin{document}

\maketitle

{
\commentAGI{
\abstract This note revisits the concepts of task and difficulty.
The notion of cognitive task and its use for the evaluation of intelligent systems is still replete with issues. 
The view of tasks as MDP in the context of reinforcement learning has been especially useful for the formalisation of 
learning tasks. However, this alternate interaction does not accommodate well for some other tasks that are usual in 
artificial intelligence and, most especially, in animal and human evaluation. In particular, we want to have a more 
general account of episodes, rewards and responses, and, most especially, the computational complexity of the 
algorithm behind an agent solving a task. This is crucial for the determination of the difficulty of a task as the 
(logarithm of the) number of computational steps required to acquire an acceptable policy for the task, which 
includes the exploration of policies and their verification. We introduce a notion of asynchronous-time stochastic 
tasks. Based on this interpretation, we can see what task difficulty is, what instance difficulty is (relative to a 
task) and also what task compositions and decompositions are.}
\commentARXIV{\abstract We establish a setting for asynchronous stochastic tasks that account for 
episodes, rewards and responses, and, most especially, the computational complexity of the 
algorithm behind an agent solving a task. This is used to determine the difficulty of a task as the 
(logarithm of the) number of computational steps required to acquire an acceptable policy for the task, which 
includes the exploration of policies and their verification. 
We also analyse instance difficulty, task compositions and decompositions.
}
}

\section{Introduction}

\commentAGI{There is an increased interest in artificial intelligence evaluation, motivated by recent breakthroughs produced by new technologies, and also because of an urging pressing of characterising the abilities of machines, so that we can have a better account of their implications in the job market and the potential risks.} In the context of universal psychometrics \cite{upsychometrics2}, defined as the evaluation of cognitive features of humans, non-human animals, computers, hybrids and collectives thereof, the notion of `cognitive task' was introduced and formalised, but several issues still require further development, such as the associated concepts of task difficulty and task breadth (or alternative concepts such as composition and decomposition). 

\commentAGI{In this paper, we realise that many tasks in artificial intelligence, human psychometrics and animal cognition do not fit well within the formalism of (PO)MDP, especially with the concept of `transition function'. With the help of some examples of cognitive tasks, we identify several features that a proper notion of cognitive task should incorporate. It is important that we realise that the evaluation setting does not need to be defined in terms of the way particular approaches solve the problem (which can still be approached through a reinforcement learning approach using a MDP formalism). What we see is that the alternate finite-state view of MDP based on transition functions makes it difficult to understand how some simple tasks, such as response time, can be accounted for, and most especially, when we want to analyse the computational complexity of the space of policies, in order to derive notions such as task difficulty.}

In the case of using formalisms that rely on transition functions such as (PO)MDP (for discrete or continuous cases), the notion of computational cost must be derived from the algorithm behind the transition function, which may have a very high variability of computational steps depending on the moment: at idle moments it may do just very few operations, whereas at other iterations it may require an exponential number of operations (or even not halt). The maximum, minimum or average for all time instants show problems (such as dependency on the time resolution for which the steps of the algorithm should remain fairly constant, or the use of space with finite states). Also, the use of transition functions differs significantly in the way animals (including humans) and many agent languages in AI work, with algorithms that can use signals and have a control of time through threads (using, e.g., ``sleep" instructions where computation stops momentarily). 
Of course, we are not saying that it is impossible to find modifications of MDP to accommodate all this, but we are going to see 
a different formalism, based on probabilistic (Turing) machines with a special ``sleep'' instruction.

The other important thing is the notion of response, score or return $R$ for an episode. Apart from relaxing its functional dependency with the rewards during an episode, to account with a goal-oriented task, we consider the problem of commensurability of different tasks by using a level of tolerance, and deriving the notion of acceptable policy from it. While this seems a cosmetic change, it paves the way to the notion of difficulty ---as difficulty does not make sense if we do not set a threshold or tolerance--- and also to the analysis of task instances.

After these instrumental accommodations, we are ready to derive the computational steps taken by an algorithm during a task. This is crucial for the notion of difficulty. With this representation, the straightforward idea of difficulty as search effort is used, whatever the kind of search is (``intellectual'', ``evolutionary'' or ``cultural'', as Turing distinguished \cite{turing1948intelligent}).   
Difficulty is just the logarithm of the computational steps that are required to find the optimal policy, including trying several possible policies and verifying them. This is in accordance with Levin's universal search  \cite{Levin73,Li-Vitanyi08}, the notion of information gain \cite{HernandezOrallo00d} and the interpretation of the ``minimal process for creating 
[something] from nothing" \cite{mayfield2007minimal}.  
\commentAGI{However, we have to be very careful that when an agent interacts with the world or a task, this task can give hints and reinforce the search process. How all this is set makes a big difference, especially for the interpretation of verification (for instance, in Levin's search, verification is simply the execution of the algorithm to check the output). It is insightful to see that in some tasks, the agent can just find policies such as \natlangprog{do what I have seen}, \natlangprog{do a Monte Carlo approach} and \natlangprog{learn from the examples} instead of the `ideal' specific policy for the problem. These policies (or meta-policies) may require fewer computational steps during the search and may lead to acceptable policies, even if the code for the search has to be counted in the description of the policy.
}

\commentAGI{
The notion of difficulty for tasks is usually applied to this generation of a policy for the task, either by evaluation or through learning. This is very different to the computational complexity of the problem. For instance, one thing is to learn a function that sorts a string and another thing is to analyse whether a certain algorithm (or any algorithm whatsoever) can sort a string in a number of steps that is polynomially related to the size of the string. Of course, we can ask about the computational complexity of {\em learning} a sort function from examples, but in this cases we need to consider several factors such as (1) the desired sort function in terms of accepted level of error, (2) what the minimum efficiency requirement for the policy is, (3) how many examples are needed and (4) how much time is needed. Some of these questions have been solved by learning theory, and settings such as PAC learning.
}

In addition, the notion of task instance difficulty is more controversial, as it usually assumes that it is relative to the task (e.g., `30+0' is an easy instance of the addition task) or even to the policy (e.g., `sort {\sffamily gabcdef}'' is a very easy case for a particular sorting algorithm). Note that average-case complexity in complexity theory refers to how many computational steps are employed to solve a set of instances (with a distribution) given a particular algorithm ---or for every possible conceivable algorithm. But one question that is not usually made is: How can we say that `sort {\sffamily gabcdef}' is easier than `sort {\sffamily gdaefcb}' without setting an algorithm or the definition of a distribution of algorithms? The key is to analyse the distribution of policies and the resources they require. Of course, this must be done relative to the task with a large (or infinite) number of instances. 
\commentAGI{We will see that otherwise (if we just focus on one instance or a small set of instances), this does not make sense, as we can just rely on memorising the policy with a lookup table.}

The paper is organised as follows. 
Section \ref{sec:tasks} starts with an example and tries to identify the features and requirements that a universal psychometric task should have to be a good evaluation task. Then it introduces a formalism, as general as possible, for this.
Section \ref{sec:difficulty} investigates the notion of task difficulty, and the necessary notions of effort (based on length and computational steps) and acceptability (using a tolerance level). 
Section \ref{sec:instance} discusses whether the notion of task difficulty can be inherited for instances. Then we move to 
the notions of task composition and decomposition and their implications, and whether this allows for the definition of response curves that may be used for adaptive tests. 
%
%
Section \ref{sec:verification} introduces a variant of Levin search that takes the stochasticity of tasks into account and includes a new term into $Kt$, which is based on the number of repetitions that are needed to verify that a policy is $\epsilon$-acceptable with some given confidence $1-\delta$, \`a la PAC (Probabilistic Approximate Correct). 
%
%
Section \ref{sec:conclusions} closes the paper with some comments about the related work and a few open questions and directions.


\section{Tasks, trials and responses}\label{sec:tasks}

Cognitive evaluation is performed through instruments, known as cognitive tests, which are composed of cognitive tasks. Consequently, we need to have a clear view of what a task is and how they can be compared. In \cite{upsychometrics2}, tasks are defined as interactive processes with asynchronous time where the final response is not {\em necessarily} a function of rewards. However, tasks are still based on transition functions and ---partly because of this--- there is no clear handling of idle times to define a proper notion of computational steps. In addition, it is unclear what happens if there is repeated testing on the same agent, and also if the agent has been gone through a previous training stage or not. 
Despite some extra notational burden, in this paper we will try to be explicit about all this.

\commentAGI{
\subsection{Example}

What do Talon the dolphin in Florida Keys \cite{jaakkola2005understanding} and Ana the sea lion in Valencia \cite{abramson2011relative} have in common? Both have been tested about their ability to judge relative quantity, a task that is usually referred to as ``relative numerousness", ``relative numerosity" or ``relative quantity judgment''. Talon the bottlenose dolphin, for instance, was repeatedly tested with two different quantities such as the two shown in Fig.~\ref{fig:numerousness}, and was given a reward if selected the lesser amount. 

\begin{figure}[ht]
\vspace{-0.3cm}
\centering
\hspace{1cm}
\includegraphics[width=0.35\textwidth]{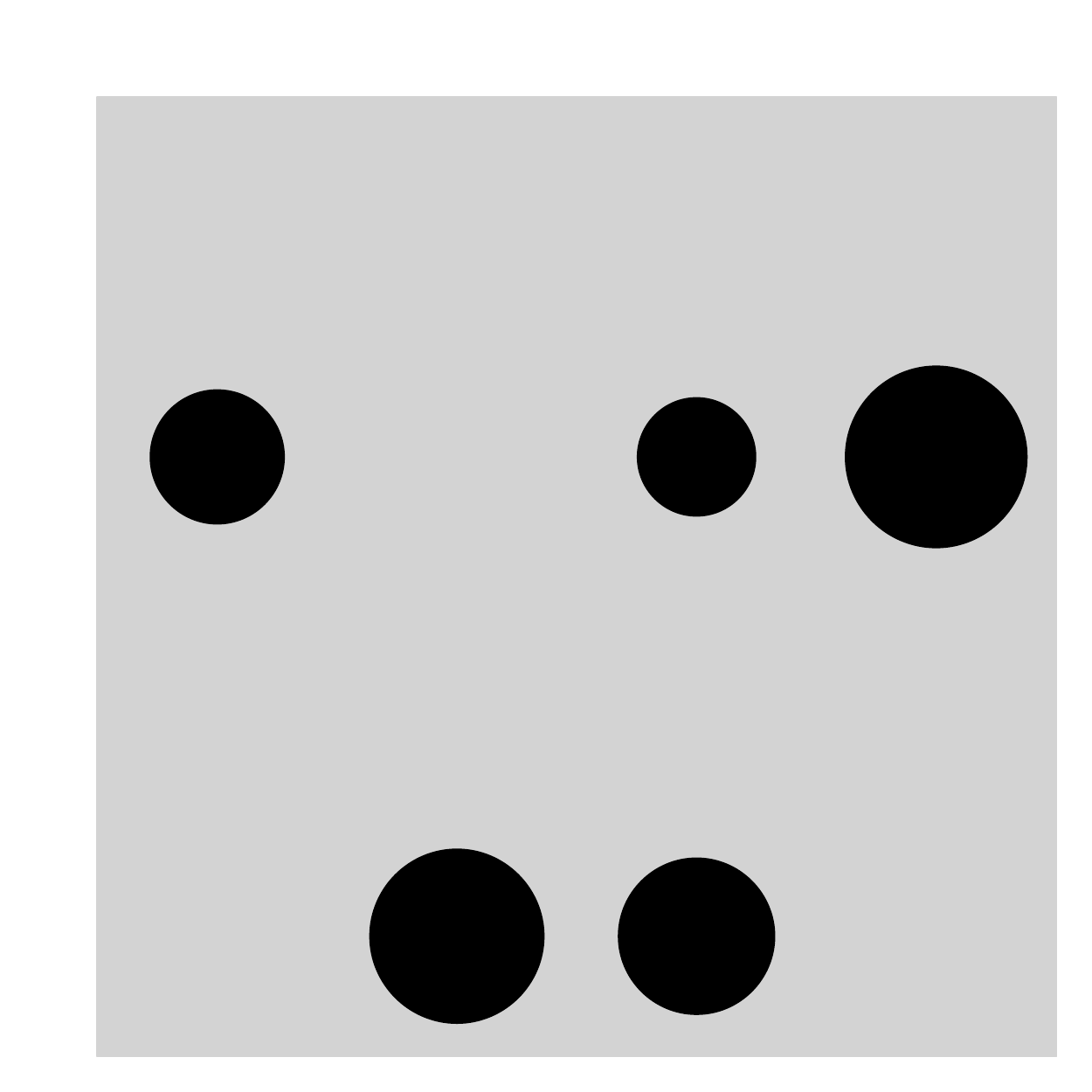} \hfill 
\includegraphics[width=0.35\textwidth]{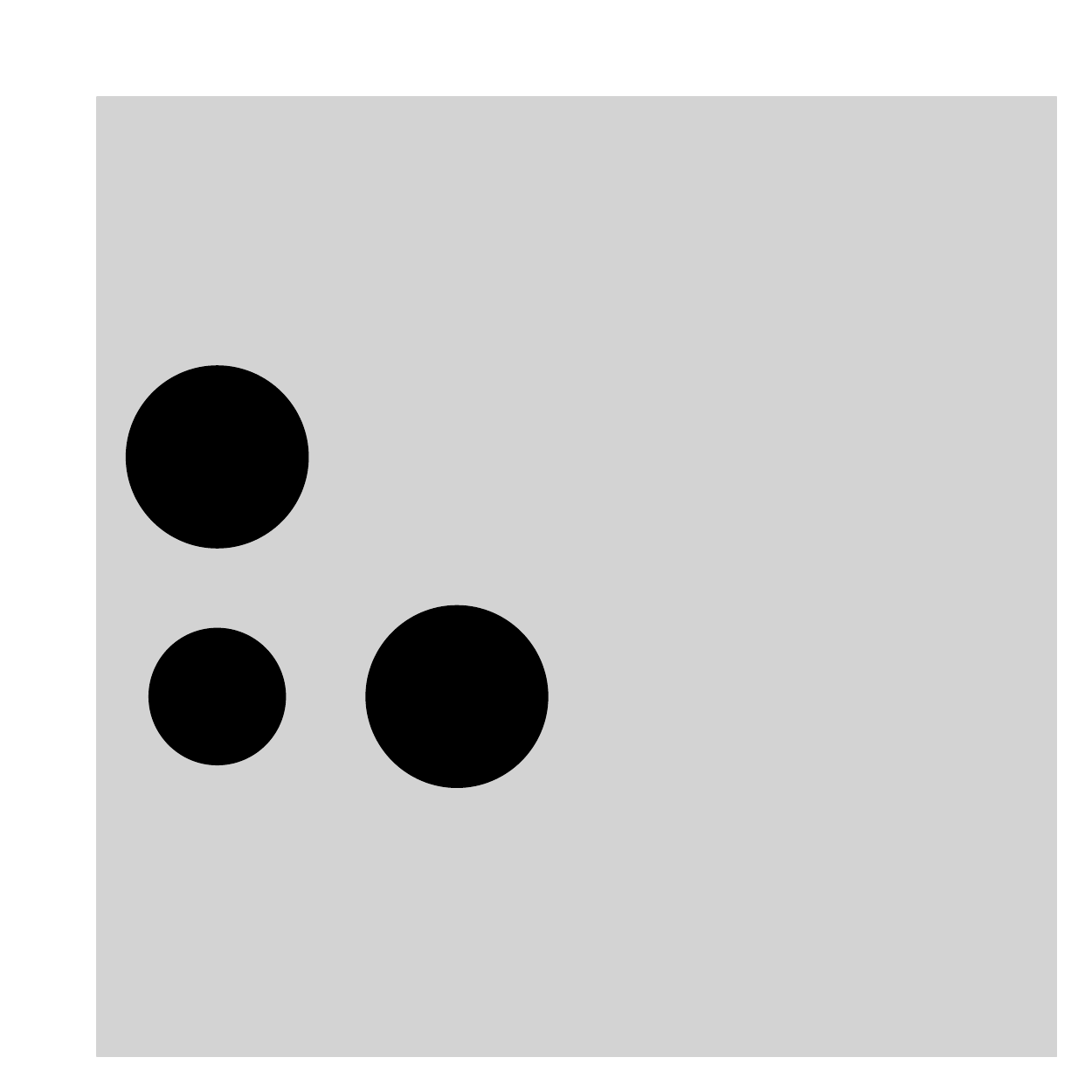} \hspace{1cm}
\vspace{-0.3cm}
\caption{An example of a `relative numerousness' task, where two boards are shown with a different number of dots. The size of the dots should not matter for the quantity. Left: a panel with 5 dots. Right: a panel with 3 dots.}
\label{fig:numerousness}
\end{figure}

Apart from cetaceans, many other studies about ``relative numerousness" 
have been conducted in the area of comparative psychology, including angelfish, bears, capuchin monkeys, squirrel monkeys, cats, chimpanzees, coyotes, gorillas, hyenas, orangutans, pigeons, salamanders, sea lions and elephants (see, e.g., \cite{abramson2011relative,perdue2012putting,vonk2012bears}, to links to some of these studies).

The interesting thing about this example is, on one hand, that it has been applied to many different kinds of animals, including humans of different ages (needless to say that the task is easy for adult humans that are allowed to count). On the other hand, it is relatively easy to write a computer program that solves this task perfectly, using image recognition and simple counting. This example will serve as a running example to illustrate some issues of tasks: level of completion, stochastic character, sequentiality, training stage, etc. Also, we will use it as a good example of whether and how difficulty can be determined formally, independently of the population results.

Other tasks (originating from psychometrics, comparative psychology or artificial intelligence) will be used in what follows and will be described in more detail if needed. For instance, we will use letter series or Raven's progressive matrices (as in IQ tests), response time, mazes, playing Pacman, English-Spanish translation, simple imitation (action equal to most recent observation), eidetic memory 
, sudokus and addition. These tasks are summarised in Table \ref{tab:tasks}.

\begin{table}
{\small 
\begin{center}
\begin{tabular}{p{1.0cm}p{3cm}p{0.7cm}p{5cm}p{5.5cm}}
 Id         & Name                          & Stoc & Description                                                             & Instances and Generation                                      \\ \hline
 $\mu_{num}$      & Relative numerousness         & yes  & Choose between left/right panels the one with fewest number of dots     & Number and size of dots uniformly chosen from a range. \\ 
 $\mu_{RPM}$      & Raven's progressive matrices  & yes  & Choose the option that better matches the matrix                        & A finite set of problems, uniformly chosen.     \\ 
 $\mu_{Ctest}$    & $C$-test                      & yes  & Find the continuation of a letter series                                & The difficulty of the sequence is uniformly chosen.     \\ 
 $\mu_{response}$ & Response time                 & yes  & Press left/right button when and as the signal indicates                                   & A uniform distribution of delays from a range. Left/right uniform too.  \\ 
 $\mu_{maze}$     & Maze                          & yes  & Go from start to exit in a maze                                           & A random generator of solvable mazes with variable proportion of walls.           \\ 
 $\mu_{pacman}$   & Pacman                        & yes  & Eat all dots without being eaten by some ghosts                         & Ghosts move with some patterns but stochastically.     \\ 
 $\mu_{trans}$    & Translation                   & yes  & Translate a text from English to Spanish                                       & Texts taken from a large finite corpus.          \\ 
 $\mu_{imit}$     & Simple imitation              & yes  & Repeatedly perform the action equal to most recent observation          & Observation chosen uniformly from a finite set  \\ 
 $\mu_{guess}$    & Guess action sequence         & yes  & Actions are guessed until match (with reward), then another action  & Sequence chosen uniformly from a finite set     \\ 
 $\mu_{eidetic}$  & Eidetic memory                & yes  & Remember a sequence of numbers that have only been shortly shown        & Various exposition times and sequences          \\ 
 $\mu_{srote}$    & Short constant string         & no   & The agent must output the string. Correct string is shown afterwards    & Always the same text for all instances          \\
 $\mu_{lrote}$    & Long constant text            & no   & The agent must output the string. Correct string is shown afterwards    & Always the same text for all instances          \\
 $\mu_{sudoku}$   & Sudokus                       & yes  & A 9$\times$9 sudoku                                                      & Consistent puzzles from a random generator                 \\  
 $\mu_{add}$      & Addition                      & yes  & Addition of two natural numbers                                         & Numbers chosen uniformly from a range.          \\  \hline
\end{tabular}
\caption{Some illustrative tasks that can be used to reason about some of the concepts discussed in this paper. The column `Stoc' indicates whether they are stochastic or not.}\label{tab:tasks}
\end{center}
}
\end{table}

}

\commentAGI{

\subsection{Features of a task}\label{sec:features}

Having a look at the `relative numerousness' and other tasks, we need to consider several features (some of them present in comparative psychology, psychometrics, reinforcement learning, etc.):

\begin{itemize}
\item Tasks can be administered in one or more trials. There is a result or response $\Response$ at the end of a task trial.
\item As trials can be repeated, if the system is not reinitialised after each trial, we have a cumulative evaluation of the task. Its evolution is measured in terms of the number of trials or attempts $\nu$. 
\item Asynchronous time: many tasks in psychometrics require time to be continuous or to be actual time. For instance, the response time task or a real-time Pacman requires the use of time. Note that there is a long tradition of discrete time in AI, especially in reinforcement learning, although the use of continuous time environments has also been studied in the areas of intelligent control and also in various kinds of reinforcement learning \cite{shelton2014tutorial}. 
We are just in favour of using asynchronous discrete time. The crucial point is that actions and observations from the environment are not alternating.
\item Trials have a limited time $\tau$. Performance depends on this limited time.
\item Interaction is given by discrete structures, but not bounded, i.e., we will consider algorithmic actions and algorithmic observations. In other words, actions and observations are complex structures that cannot be represented with a finite set of actions and observations. For instance, in the `relative numerousness' or `Pacman' tasks we can assume a finite grid of points up to some given resolution, but for an English-Chinese translation task inputs and outputs are, in principle, not bounded.
\item States are algorithmic. There is no finite set of states. The Markov property is not assumed either. Tasks are usually non-ergodic (it is the repetition of several task instances what makes learning possible).
\item Tasks (and subjects) are stochastic (if they are not stochastic ---or not very stochastic---, rote learning will be frequent). Several trials for a task can give different results. 
\item When several instances of the {\em same} task are performed they can be averaged and their expected value estimated. However, it is important to note that for {\em different} tasks, the aggregation of the response of different tasks (e.g., an average) might not make sense (if the values are not commensurate). When using different tasks, if they are to be aggregated nonetheless, the final score of a test can depend on tolerance levels $\epsilon$ over the responses. Only if these are seen in terms of similar difficulty, the numerical aggregation (and the notion of task composition) can become  meaningful. 
\item Rewards are a kind of transmitting supervision during a trial. They may exist or not, and may be linked or not to the response $\Response$. Difficulty will of course be affected by the (non-)existence of rewards. In any case, it is important to clarify that observations can be an indirect sign of supervision too, as we are talking about interactive tasks.
\item In order to evaluate an agent, we do not need anything about the size of the algorithm behind the agent 
or the computational steps it requires, 
just whether it makes some proper actions in due time. The size of their algorithms and their computational steps are important for the calculation of the difficulty, as we will see in the following section. 
\end{itemize}

\noindent 
The relation between repeated trials and rewards deserves some further discussion. If a task has only one trial (or the agent is reinitialised after each trial) and does not have intermediate rewards as in reinforcement learning, then the system must be necessarily predefined and specialised for that task. This is what most AI applications are conceived for. In animals, some tasks trigger an innate behaviour and can be measured in these circumstances. For instance, many animals can choose the board with the {\em highest} number of peanuts or fish without any training at all (and no intermediate rewards indicating whether it is doing right or not). Of course, the innate behaviour takes place because the task (or a similar one) has appeared many times in the evolutionary history of the species. 
However, many other tasks require some training, and this can be done in animals and in AI systems. In animals, rewards can be given at the end of an episode or during the episode. Similarly, in AI, rewards (or payoffs) are given at the end of an episode (e.g., in game theory) or during the episode (e.g., in reinforcement learning, with the reward function).  Even if these two approaches exist for training, when we focus on measuring capabilities and skills, it is usual that intermediate rewards are no longer used, as their effect is more difficult to control and understand. In fact, this is not actually a distinction between animal cognition evaluation and AI evaluation, but a distinction given by the purpose. For instance, in (video)games, it is usual that there is an intermediate reward in the form of points for a score, apart from the goals of each stage or the whole game.

From the above, it seems that for many tasks where the agent is not innate or preprogrammed for, in order to measure abilities and not specific task performance, we need to consider tasks that are both stochastic and with several trials. Several trials allow the system to be trained for the task, where the use of stochastic tasks ensures that the system does something different from rote learning (note that a large set of items chosen randomly is a stochastic task). In the case of several trials, it is important to consider that for animals (including humans) and some AI systems, reinitialisation is not possible, so we have to take into account that the realisation and result of previous trials have effect on subsequent trials. 
Only some tasks can avoid this effect. In fact, some tasks used in IQ tests are usually designed in such a way that there is no much interference between one exercise and the rest (in fact there is no learning or specialisation), although this effect can never be ruled out completely. Finally, in adaptive tests, dependency between trials  is not only existent but characteristic. Actually, this dependency is exploited. The most general account of a task would be to consider that they are adaptive (i.e., they have memory as well), and non-adaptive tasks would be a special case. Even for a single task, we can have an adaptive test, provided we have a measure of instance difficulty or some other feature that we can use to change the distribution of instances of a task. We will deal with the  issue of instance difficulty later on.

}

\commentAGI{
\subsection{Asynchronous-time Stochastic Tasks}
}

\commentAGI{
Now we are going to give a more formal account about how to define general tasks computationally, which comply with the features in section \ref{sec:features} above. We want interactive tasks such that, in an episode, agent and environment can exchange inputs and outputs at any time. We will first choose asynchronous time for it, as this is needed in some tasks such as `response time' and other real-time problems. Apart from its need in these types of tasks, there are additional reasons for using asynchronous time  in reinforcement learning
\cite{shelton2014tutorial},  
and artificial life \cite{cornforth2003artificial,fates2013guided}. 
Even in cases where the task is alternating (e.g., a chess match), it is important to consider the time for each turn and the thinking time (one can think while the opponent is thinking, and both thinking times have to be considered).

Synchronous (or more precisely, alternating discrete time) interactive machines are based on a transition function, which is applied at each time point to change the state. The most common example is (PO)MDP. The transition function takes a state, and observation and a reward and produces an action. 
It goes from state to state indefinitely (even if it remains in the same state forever, there is some computation to apply at each time moment, the transition function). Transition functions can have access to the environments's memory. In this case, if the memory is not bounded we have an infinite number of states (no longer an MDP).  
In any case, even with a finite number of states because of the stochastic character there might be a different number of computational steps taken for each transition (there might even be some transitions that do not halt).

Asynchronous environments are not continuous-time POMDPs, which are based on transition functions and are handled with differential equations.  
In fact, synchronous environments are a special case of asynchronous environments where the environment  waits for the agent's action to issue observations and rewards. Intermediate rewards during the episode are also considered but, unlike synchronous environments, the correspondence of the total result as a discounted sum of rewards is not possible in general. In fact, the number of rewards per unit of time is not limited, so the final function that maps rewards to a result may be very varied (and it is part of the definition of the task). This is similar to the way rewards were defined in \cite{upsychometrics2}, an internal thing given to the agent, whereas the score or response for the episode was an external thing not necessarily given to the agent. 



}

Let us \commentAGI{now} give the definition of asynchronous-time interactive systems. 
In an asynchronous-time interactive system, there is a common time (which can be discrete or continuous, and can be virtual or real). Time  will be shared by all systems that interact. %
An interactive system is a machine with a program code, a finite internal discrete memory, one or more finite read-only discrete input (tape) and one or more finite write-only discrete outputs (tape). 
Agents and environments use the above definition and are asynchronous-time interactive systems. 
The inputs of agents are called observations and the outputs are called actions. 
As special features, these machines have access to a read-only time measurement and a source of randomness (either by an additional random instruction or a random tape). 
The programs for tasks and agents are constructed with a set of instructions that, if memory were infinite, would make the machine Turing-complete, and ultimately equivalent to a Turing machine, denoted by its program over a reference universal prefix Turing machine $U$.
This makes this definition very close to probabilistic Turing machines\commentAGI{\footnote{Probabilistic Turing machines with finite tapes (except the random tape) are like ``probabilistic linear bounded automata". 
This is exactly the type of computers we are used to and the ones we are able to build with the current paradigm. Note that this is different to subrecursive programming systems and other models of computation where it can be determined whether programs terminate, i.e., and even what they will compute.}} ---which are not exactly the same as non-deterministic Turing machines. In a probabilisitic Turing machine, only one course of action is taken, and no parallel computation is performed to keep all the alternative courses of actions. In fact, computable stochastic processes are usually associated to probabilistic Turing machines, and not to non-deterministic Turing machines.

\commentAGI{
We have already said that the machine will have access to a random source (through an instruction or an extra random tape).  
Some animals (e.g., flies, preys) behave in a random way to avoid being predated, and this behaviour does not require a very complex mechanism, 
just a few neurons being triggered on some environmental magnitudes acting as random number generators.
There is also access to time, 
which can be physical time, an approximation or a virtual time. 
But most importantly, for} \commentARXIV{For} the purpose of the analysis of computational steps, we consider that the machine will be able to stop momentarily, until a given time, through an instruction or special state 
{\sffamily sleep(t)}, which sets the machine to sleep until time $t$. \commentAGI{During the time the machine is sleeping, no operation is performed.} 

Some tasks will also have intermediate rewards. Rewards are just given through another extra tape, and are interpreted as a natural number. 
Rewards are optional. In case they exist, the result of an episode may depend on the rewards or not. This is important, as the general use of rewards in reinforcement learning, especially with discounted reward or through averaging gives the impression that the final result or response of an episode must always be an aggregation. For instance, in a maze, an agent may go directly to the exit and may require no reward. On the contrary, a more sluggish agent may require more positive indications and even with them cannot find the exit. Rewards can be just given to help in the finding of the solution, which does not mean that the higher the rewards the higher the results. 
Finally, the agent is able to see the result or score of an episode (a rational number) at the end through another special tape. 
A final reward can be given instead of or jointly with the result. 

While this is certainly more complex than other models of interactive machines, it accommodates the intuitive notion of task in many natural and artificial scenarios.

\comment{
$\Response$ of course can depende on doing the action at a proper time..


\begin{definition}\label{def:intsystem}
An interactive probabilistic machine is defined as a tuple $\left\langle {{\cal{T}}, \cal{S}}, {\cal{O}}, {\cal{I}}, \dot{s}, \dot{o} \right\rangle $, where 
${\cal{T}}$ is the time space,
${\cal{S}}$ is the state space,
${\cal{O}}$ is the output space,
${\cal{I}}$ is the input space,
$\dot{s}(s,i)$ is a transition rate function: ${\cal{S}} \times {\cal{I}} \Longrightarrow \Delta {\cal{S}} $, and
$\dot{o}(s,i)$ is an output function: ${\cal{S}} \times {\cal{I}} \Longrightarrow {\cal{O}} $.
\end{definition}

\begin{definition}\label{def:cogtask}
A cognitive task is defined as a physically computable interactive system $\left\langle {\cal{T}}, {\cal{S}}, {\cal{O}}, {\cal{I}}, \dot{s}, \dot{o} \right\rangle$ following definition \ref{def:intsystem} (with outputs being observations and inputs being actions), and a score transition function $\dot{u} : {\cal{U}} \times {\cal{S}} \times {\cal{O}}
\Longrightarrow \Delta {\cal{U}}$, with ${\cal{U}}$ being a bounded set in $\mathbb{Q}$.
\end{definition}

\begin{definition}\label{def:interface}
Given a task $\mu$ = $\left\langle {{\cal{T_\mu}}, \cal{S_\mu}}, {\cal{O_\mu}}, {\cal{I_\mu}}, \dot{s}_\mu, \dot{o}_\mu \right\rangle$ and an agent $\pi$ =$ \left\langle {\cal{T}}_\pi, {\cal{S}}_\pi, {\cal{O}}_\pi, {\cal{I}}_\pi, \dot{s}_\pi, \dot{o}_\pi \right\rangle$, an interface $\phi$ between them is a tuple of 
mappings  $\left\langle \phi_{\cal{T}}, \phi_{\cal{A}}, \phi_{\cal{Z}} \right\rangle$,
with 
${\phi_{\cal{T}}}: {\cal{T_\pi}} \Longrightarrow {\cal{T_\mu}}$ being the time mapping, 
${\phi_{\cal{A}}}: {\cal{O_\pi}} \Longrightarrow {\cal{I_\mu}}$ being the action mapping, 
${\phi_{\cal{Z}}}: {\cal{I_\pi}} \Longrightarrow {\cal{O_\mu}}$ being the observation mapping.
Given these mappings, we will take the domains of the task as a reference,
just making 
${\cal{T}} = {\cal{T_\mu}}$ the time domain, 
${\cal{A}} = {\cal{I_\mu}}$ the action domain, and
${\cal{Z}} = {\cal{O_\mu}}$ the observation domain.
\end{definition}

As we see in definitions \ref{def:cogtask} and \ref{def:interface}, the concept of `reward' does not exist as a separate entity. In case rewards are used, they are considered part of the task's output and part of the agent's inputs.
Note that in reinforcement learning, rewards and observations are considered in a separate way.
One straightforward way to work with rewards is to define task's outputs (${\cal{Z}} = {\cal{O_\mu}}$) as a tuple $\left\langle {\cal{L}}, {\cal{R}} \right\rangle$ where the second term is understood as a reward as in example XXXX
. Obviously, in order to use rewards, both task and agent have to treat part of their outputs and inputs (respectively) as a reward.

Note that the score function $\dot{u}$ is not used in the interface. Consequently, the score is not visible to the agent and cannot be calculated from the observation (alone) in general. Any appreciation the task wants to convey to the agent during the interaction must be done using observations (and possibly embedded rewards). In fact, it is important to see that score can be completely independent from rewards.
This may sound counterintuitive if compared to typical agents and environments in reinforcement learning, where performance is calculated as an aggregation of rewards. However, it is important to emphasise that reinforcement learning is about `learning' and many cognitive tasks are not about learning. In addition, on many occasions, even in learning scenarios, the performance cannot be derived as an aggregation of rewards. 

}

\commentAGI{
\subsection{Trials and results}
}

We consider tests that are composed of tasks (also called environments), usually denoted by $\mu$,  
and are performed by agents (also called policies or subjects),  
usually denoted by $\pi$.

The expected value of the response, return or result 
of $\pi$ in $\mu$ for a time limit $\tau$ is denoted by 
$\EResponse^{[\tau]}(\pi,\mu)$. 
The value of $\tau$ will be usually omitted as it is understood that it is part of the description of the task $\mu$. 
The $\EResponse$ function always gives values between 0 and 1 and we assume it is always defined. 
If the agent goes into a non-halting loop and stops reacting, this is not perceivable externally and may even lead to some non-zero $\EResponse$ (from the previous actions or because of the type of task). 


Now we need to extend the notation of 
$\EResponse(\pi, \mu)$ 
to consider several instances of the same task. Each attempt of a subject on one of the task instances is a trial or episode. 
$\EResponse^{[\TAUC \nuT\mapsto\nuE]}(\pi, \mu)$ 
returns the expected response of $\mu$ per trial with 
$\nuE$ consecutive episodes or trials by the same agent $\pi$ {\em without reinitialisation\footnote{Note that, if the test is not adaptive, instances have no memory, as they start from scratch. 
This `stochastic repeatability' is related to some other conditions (e.g., ergodicity) that are sometimes imposed or assumed on tasks where a pattern or some properties can endure indefinitely.}}. So actually it is not the same $\pi$ each time, if the agent has memory. 
$\nuE$ refers to the evaluation trials, which 
are used for the expected response (which is an average of all the evaluation trials and not a sum).
Note that the expected response is given because $\pi$ is non-deterministic and may lead to different situations from the very beginning. The distribution of what each instance of a trial will look like is inside the stochastic task. 
According to the task, the same instance can appear more than once, as in a sample with replacement. 
As the task can have memory, we can also have some tasks that are really working as if a no-replacement sampling were taking place. In order to do that, the task itself must keep track of the instances that have appeared or must use some kind of randomised enumeration. Also, tasks can be adaptive. In other words, instead of talking about sampling with or without replacement, it is the definition of the task that defines this. 

\commentAGI{
No waiting time or stop is considered between trials. If a task requires some resting time between trials, then this has to be included in the very trial and not in between trials.
}


With 
$\EResponse^{[\TAUC \mapsto 1]}(\pi, \mu)$, or simply $\EResponse(\pi, \mu)$ 
we denote that there is 
only one episode or trial (no repetitions). For instance, many tests are of this kind if items are completely unrelated, so each item has no influence on the following ones, although it is more applicable when we consider that the agent has no memory (or is reinitialised between trials). 
In general, especially if the items are related, for every 
$\nuE > 1$, we have that 
$\EResponse^{[\TAUC \nuT \mapsto \nuE]}(\pi, \mu) \neq 
\EResponse^{[\TAUC \nuT \mapsto 1]}(\pi, \mu)$ unless the agent has no memory between episodes.

\commentAGI{
Tasks with high values of $\tau$ will imply that episodes are long, while high values of $\nu$ mean that we make many repetitions. Note that some abilities are related to good results after very few repetitions (i.e., to understand a concept fast). This speed is understood in many ways, but one is clearly how many `examples' or `instances' are needed. Note that many machine learning techniques require many examples (e.g., deep learning \cite{Arel-etal2010}), many repetitions (e.g., Q-learning \cite{watkins1992q}) or large $\tau$ (e.g., AIXI \cite{veness2011monte}). 
In the previous example about the `relative numerousness', $\tau$ is not very relevant as the task displays the boards (or panels or dishes) for a few seconds ($\tau$ may be 5 physical seconds). However, $\nu$ is important, and we usually require a number of training trials (so that the animal can learn the task) and then a series of test trials. 
}

\commentAGI{
Note that if the {\em task} has no memory, this does not allow for an evolving distribution (e.g., a kind of task first and then switch to other tasks, or some kind of cumulative or adaptive tasks). Tasks with memory would be useful for adaptive tests. In this paper, unless stated otherwise, all constructs are valid with tasks with memory, even if we do not explore adaptive tests, just the fundamentals (such as difficulty, which is required for adaptive tests). 
}

\commentAGI{

\subsection{Examples}

In Fig.~\ref{fig:numerousness}, we saw an example of the `relative numerousness' task $\mu_{num}$.
This can be seen as a stochastic task class where the agent sees two rectangular grids (representing plates) where we have some black spots on it.
The action is just choosing left or right. 
If the choice is correct, the agent receives a response (and reward) of 1. Otherwise, it receives 0. 

For this task ($\mu_{num}$ in Table \ref{tab:tasks}), we have $4 \times 4$ `cells', with the number of dots in each panel going uniformly from 1 to 16. The size of each dot is uniformly distributed between 0.2 and 1, with 1 being the diameter of the cell. In case the two panels had exactly the same number of dots, the pair would be discarded and a new one would be generated. The use of different dot sizes is used to prevent subjects from choosing the panels exclusively (or mostly) by their overall darkness (if there are more dots and all are equal sized then the panel is always darker overall). In many studies, 80\% success rate is considered as a level where the subject is considered to perform the task successfully. 

It is relatively easy to implement an agent that processes the image, recognises the shapes and counts the dots. However, we are interested in seeing that it is also possible to score well in this task with an agent that does not count at all. This agent, $\pi_1$ performs a Monte Carlo approach and (virtually) throws $n$ points randomly inside the panel. It calculates the darkness of the panel as the percentage of points that are black (i.e., it is inside a dot). At the end, the darkness of both panels is compared and the least dark is chosen. If $\pi_1$ uses $n=100$ points for each panel, the agent is able to score $0.8675$. 
 Note that even if there are $(4^2-1) \times (4^2-2) = 210$ different number comparisons, the possible cell locations of the dots and their different sizes make a virtually infinite number of different instances.  Different results are obtained if the number of points of the Monte Carlo method is changed. For instance, if $\pi_2$ only uses $n=50$ then $\EResponse(\pi_2, \mu_{num}) = 0.8495$. Still, if $\pi_3$ only uses $n=10$ then $\EResponse(\pi_3, \mu_{num}) = 0.746$. Actually, with just one point, $\pi_4$ can still do significatively better than random: $\EResponse(\pi_4, \mu_{num}) = 0.575$. Clearly, the computational cost decreases from $\pi_1$ to $\pi_2$.


The response time task ($\mu_{response}$ in Table \ref{tab:tasks}) is an interesting task to analyse. We could have a policy $\pi_1$ that is constantly checking the input to see if a response is needed. Assuming very high speed (e.g., it can check the input, process it and see whether it has to react or not one million times per second), this $\pi_1$ would score almost perfectly. However, it would also use many computational steps. Another algorithm $\pi_2$ could just check 10 times per second (by using the instruction {\sffamily sleep(t)}, with $t=0.1$s), and get a reasonable good result with much less computational cost than $\pi_1$ (it is not exactly 100,000 less because when the signal is not there the instructions to be executed are expected to be fewer than when the signal is there).

These two examples stress the issue of computational complexity and how it is interpreted in asynchronous tasks.

}

\section{Task difficulty}\label{sec:difficulty}

\commentAGI{
The first thing to clarify about difficulty is whether we apply it to the {\em generation} of the policy or the {\em application} of the policy. 
The generation phase can be innate (by programming or nature) or acquired (through training or learning).
In comparative psychology and artificial intelligence it is usual to have these two phases.
It is very important to determine which phase we are referring to when talking about difficulty. For instance, if we evaluate the ability of an animal of being able to do a task that involves counting, what we want to know is whether the animal can acquire this ability. 
If we evaluate the ability of making calculations (e.g., addition), we are clearly assuming that the system already has the algorithm for addition, and we are just examining how well they do. This is clearly the case in many specific-task evaluation, such as driving a car, game-playing, etc.
The confusion comes because in both cases we will evaluate the performance on the task in the same way.
}

\commentAGI{
Despite the same evalution, the notion of difficulty must be understood very differently. For instance, the difficulty of the generation phase usually refers to tasks with many instances (how difficult is it to learn to add from examples), while application usually refers to instances (how difficult ``3+2" is compared to ``234+998"). In this section we will focus on task difficulty, leaving instance difficulty for the next section.
For instance, in the relative numerousness task, the generation difficulty depends on how much it takes to program the algorithm for this task, the evolutionary cost of acquiring the algorithm or the learning cost of acquiring the algorithm. 
}

\commentAGI{
The difficulty of solving a stochastic task can be assessed by \cite{liu2012task} (1) looking at the complexity of the task (this is known as a structuralist approach), (2) looking at the complexity of the policy (or the resources that are required by the subject) or (3) looking at the interaction between task and subject. 
}



Our view of the generation difficulty is an ``algorithmic difficulty", which is basically the computational steps required to build the policy algorithm, which depends on the tolerance level of the task, the interaction and hints given by the task, the algorithm length, its computation cost and its verification cost. 
We now see all these components below.

\commentAGI{
\subsection{Agent resources, acceptability and interaction in asynchronous environments}
}

The first thing we will require is the length of a policy. The length of an object $x$, denoted by 
$\Length(x)$ expresses the length of a string using a binary code for the object. This function 
can be applied to tasks and agents. 
\commentAGI{There is an important thing to consider here. If a program has the ability to self-modify, as it happens with self-improvement agents, then when we measure $\Length(\pi)$ of an agent $\pi$ during a series of trials, the value might change. 
However, one program can get extremely short by moving all the code to memory. Consequently, analysing the evolution of the program during the execution of several trials is like analysing how memory is evolving, so we will just consider the program $\pi$ as it was before the evaluation.
}

\comment{
We will have this into account and introduce 
$\ELength^{[\TAUC \nuT \mapsto \nuE]}(\pi,\mu)$ for a time limit 
 $\tau$ for each trial, where the length at trial $i$ is measured {\em before} the trial $i$ begins. Note that this length is an expected response, as it depends on the stochastic task. With this, it is clear that $\Length(\pi) =
\ELength^{[\TAUC  
\mapsto 1]}(\pi,\mu)$.
}


The second thing we will require is the computation steps taken by a policy. In synchronous environments, 
one option may be to add all the steps taken for all time cycles, but this clearly depends on the resolution of the discrete time. Also, many transition functions may be just idle transitions, where the agent is just checking whether something is happening. But imagine an agent that wants to wait for 10,000 transitions. Even if very few operations are executed in each transition, these transitions count. To avoid this problem another option is to calculate the maximum, as done in \cite{HernandezOrallo-Dowe2010} with the so-called $Kt^{max}$.  
This is a very rough approximation, as one single peak can make this very large. The mean or sum do not behave better, either. 

Fortunately, here tasks are defined as asynchronous. When the agent needs to wait until a situation or time is met, if the instruction  {\sffamily sleep(t)} is used, we should not consider all these `waiting' times for the computational steps. 
With this interpretation, the expected\footnote{This has to be `expected' if we consider stochastic environments or agents.} execution steps of $\pi$ per trial when performing task $\mu$ are denoted by 
$\ESteps^{[\TAUC  \nuT \mapsto \nuE]}(\pi,\mu)$ for a time limit 
 ($\tau$) given by the task for each trial. Note that we consider all the computational steps performed by $\pi$ during all the $\nuE$ {\em evaluation} trials for this expected value (they are not added, though).  
If at any moment $\pi$ enters an infinite loop\footnote{Here, we are not concerned about halting, but rather that the number of steps is finite before the time limit $\tau$.}, then 
$\ESteps^{[\TAUC \nuT \mapsto \nuE]}(\pi,\mu)$ is infinite. 
As we are using stochastic agents and environments, it is sufficient that one possible combination of the trials leads to non-termination such that the expectation is infinite. 

The third thing is about memory requirements (space). What if a policy requires much more memory than another? 
\commentAGI{This is also important in the context of several trials if the policy requires the memorisation of the information of previous trials.
We would do similary for the internal memory used by the policy, considering that there are instructions to ask for more memory and free memory (or we can record up to where the algorithm reaches if there is an internal tape). The notation is 
$\EMemory^{[\TAUC \nuT \mapsto \nuE]}(\pi,\mu)$. 
}
In this paper we will not consider space because (1) the use of $n$ bits of memory requires at least $n$ computational steps, so the latter are going to be considered anyway and (2) steps and bits are different units.

The fourth thing is verification. When we discuss the effort about finding a good policy, there must be some degree of certainty that the policy is reasonably good. As tasks and agents are stochastic, this verification is more cumbersome than in a non-stochastic case. We will discuss about this later on in the paper.

For the moment, we will just combine the length of the policy and the computational steps, by defining 
$\ELS^{[\TAUC \nuT \mapsto \nuE]}(\pi,\mu) \triangleq \Length(\pi) + \log \ESteps^{[\TAUC \nuT \mapsto \nuE]}(\pi,\mu)$. 
Logarithms are always binary. We will explain later on why we apply a logarithm over $\ESteps$.

The fifth thing is the tolerance level of the task. In many cases, we cannot talk about difficulty if there is no threshold or limit for which we consider a policy acceptable. For instance, how difficult is a response time task? It depends on where we put the threshold. How difficult is pacman? It depends on how many points or time we want to achieve. It is true that some tasks have a response function $\Response$ that can only be 0 or 1, and difficulty is just defined in terms of this goal. But many other tasks are not binary (goal-oriented), and we need to establish a threshold for them. 
 In our case, as the return function $\Response$ goes from 0 to 1, we can take 1 as the best response and set the threshold on $1 - \epsilon$. With this we first consider the notion of acceptability. 

We define acceptability in a straightforward way. The set of acceptable policies for task $\mu$ given a tolerance $\epsilon$ is given by 

\begin{equation}\label{eq:acceptN}
\AccSet^{[\epsilon, \TAUC  \nuT \mapsto  \nuE]}(\mu) \triangleq \{ \pi \::\: \EResponse^{[\TAUC \nuT \mapsto \nuE]}(\pi, \mu) \geq 1- \epsilon \}
\end{equation}
%
%
%
%
Note that the combination of the expected value with a tolerance greater than 0 makes that the agent can do terribly wrong in a few instances, provided it does well on many others. \commentAGI{While the expected value corresponds to the mean, we could use another statistic.}



\commentAGI{
The sixth thing is the interaction and hints given by the task. This can be during the task (through rewards or other observations) or throughout several trials. During the task, algorithms can use past experience and rewards to solve the task. For instance, $\mu_{imit}$ in Table \ref{tab:tasks} can be solved by simply observing and copying, so actually the policy is an algorithm that does this. Similarly, if we have an agent that is not reset after each trial, the algorithm can just learn from previous trials. For instance, $\mu_{lrote}$ in Table \ref{tab:tasks} is solvable by an algorithm that memorises the correct string from a previous trial. 
In general, we can have many different kinds of policies: 
\natlangprog{forever do action 1, wait(1), do action 2, wait(1)}, which ignores the observations from the task completely,  
\natlangprog{forever output what $\mu$ outputs}, which uses observations but ignores previous trials, 
\natlangprog{execute code1, if result of previous trial is lower than 0.5 then execute code2 in the following trials}, which uses the results of previous trials, 
\natlangprog{execute random actions every 1 units of time. Memorise those actions that generate some change of observations. Repeat them on the following trials}, which uses the observations of previous trials, and 
\natlangprog{execute random actions every 1 units of time. Memorise those actions that receive positive rewards. Repeat them on the following trials}, which uses the observations and rewards (if there are) of previous trials. 
But some other `meta-algorithms' are equally valid, such as \natlangprog{try algorithms randomly from a given set of algorithms. If one has been good for the past five trials, use it for ever} or \natlangprog{use search heuristics of type 1 for the first 100 trials. If unsuccessful, switch to heuristics of type 2}. These are just examples of whatever algorithm that can be used, including self-improving algorithms. 
} 

\commentAGI{
The consideration of stochastic tasks is fundamental. For instance, consider the ``relative numerousness'' task $\mu_{num}$ again in Table \ref{tab:tasks}. For each instance, the solution is just `left' or `right'. The information that is needed is just one bit. If we just put one possible instance in a task (i.e., a non-stochastic task) ---and we knew that the task is not stochastic---, then just one repetition of the {\em same task instance} would be enough to find the solution, which will be very short (e.g., \natlangprog{choose the left board}). If the task can generate a great or infinite number of instances, then the possibility of rote learning is reduced, and the policy would incorporate some generalisation. 
}




\commentAGI{
\subsection{Difficulty as minimum resources}\label{sec:KandKt}
}

Having the the above issues into account we can define a first parametrised version of difficulty. 
Bear in mind that these are general expressions whose goal is to understand what a function of difficulty is. In many tasks, though, we may use a more practical (and particular) function of difficulty. 




And now we are ready to link difficulty to resources. This is usual in algorithmic information theory, but here we need to calculate the complexities of the policies (the agents) and not the problems (the tasks). 
\commentAGI{
So, our first approach is to evaluate difficulty as the length of the shortest acceptable policy: 

\begin{equation}\label{eq:K}
K^{[\epsilon, \TAUC  \nuT \mapsto  \nuE]}(\mu) \triangleq \min_{\pi \in \AccSet^{[\TAUC  \epsilon, \nuT \mapsto  \nuE]}(\mu)} 
\Length(\pi) 
\end{equation}

The use of the notation $K$ and the structure of the definition make it clear that this can be understood as a version of Kolmogorov complexity for tasks, where instead of talking of the shortest program that generates a string, we talk about the shortest program that solves the task.

%
Note that $K$ is not only parametrised with a tolerance $\epsilon$ but also with 
the number of 
evaluation trials. 
So our notion of difficulty depends on these parameters. 
We could think about letting 
be unlimitted, so we would have $K^{[\epsilon, 
 \mapsto  \infty]}(\mu)$. 
This allows programs that use several trials, so we can have a policy $\pi$ that just does \natlangprog{enumerate all possible programs and execute each of them on as many trials as needed and choose the best one for the subsequent trials}. Let us call this strategy\footnote{In a way, this strategy is like an AIXI-like algorithm \cite{Hutter05}.}, $\pi_{find-L-best}$. Assuming there is a finite acceptable policy, the length of this program $\pi_{find-L-best}$ 
could be taken as an upper bound for $K$ because this program is going to find the policy if given infinite trials, just by enumeration. For some tasks, of course, there might be other programs that could be shorter than $\pi_{find-L-best}$. For instance, in the simple imitation task $\pi_{imit}$, it is expected that the coding of the program \natlangprog{copy the observation to the action} is shorter than $\pi_{find-L-best}$. Examples for some of tasks are shown on Table \ref{tab:tasksopt}.

We can also consider $K^{[\epsilon, \infty, 
\mapsto  1]}(\mu)$, but in this case it cannot be a program that searches for the policy across several trials. For some kinds of tasks, especially those that do not give partial indications during the task, this will account for the shortest policy that gives an $\epsilon$-acceptable solution {\em without looking at the task at all}. For others, the task will provide the required information (like an input) but the interaction will be just that. For instance, in the relative numerousness task, depending on the tolerance, the Monte Carlo policy could be a good option, as it is a very short policy. 
Actually, a version of the Monte Carlo with a huge amount of points would be better, disregarding its high computational cost, since computationl steps are not taken into account. 
For the simple imitation task \natlangprog{copy the observation to the action} would still be chosen.
It may seem counterintuitive to analyse a situation with just one trial with a policy that cannot be found with just one trial (the chances are actually about $2^{-L}$), but here we are trying to measure difficulty. 
Examples for some of the taks are shown on Table \ref{tab:tasksopt}.

The problem about $K$ is that it does not take computational cost into account (this also makes it uncomputable). 
} 
A common solution, inspired by Levin's $Kt$ (see, e.g., \cite{Levin73} or \cite{Li-Vitanyi08}), is to define:
%
%
\begin{equation}\label{eq:Kt}
Kt^{[\epsilon, \TAUC  \nuT \mapsto  \nuE]}(\mu) \triangleq \min_{\pi \in \AccSet^{[\TAUC  \epsilon, \nuT \mapsto  \nuE]}(\mu)} {\ELS^{[\TAUC  \nuT \mapsto  \nuE]}(\pi,\mu)}
\end{equation}
Note that the above has two expectations: one in $\ELS$ and another one inside $\AccSet$. The interpretation of the above expression is a measure of effort, as used with the concept of computational information gain with $Kt$ in \cite{HernandezOrallo00d}. 


\commentAGI{
We first consider $Kt^{[\epsilon, \TAUC  
\mapsto \infty]}(\mu)$. With this we allow for as many trials during 
evaluation. In other words, effort can be put in finding the policy, but the policy must be efficient. Again, this would sometimes end up choosing \natlangprog{enumerate all possible programs and execute each of them on as many 
trials as needed and choose the best efficient one}. 
This happens because despite its enormous computational cost, when the trials go to $\infty$ the algorithm may finally find the particular policy and start exploiting it. As it is the expected value for the infinite number of trials that counts, this policy is efficient for an infinite number of trials. Let us call this strategy\footnote{In a way, this strategy is like an AIXItl-like algorithm \cite{veness2011monte}.}, $\pi_{find-LS-best}$, which is again of not much practical use.  Examples for some of the task are shown on Table \ref{tab:tasksopt}.

We can compare with $Kt^{[\epsilon, \TAUC  
\mapsto  1]}(\mu)$. In this case, the meta-policies such as $\pi_{find-LS-best}$ are avoided, but we have that the policy cannot take advantage of previous trials. In a way, this version is measuring difficulty when the agents have no memory (or are reinitialised).

All these options are summarised in Table \ref{tab:tasksopt}, which shows the tasks introduced in Table~\ref{tab:tasks} with the values of several difficulty functions. 
We see that in some cases, the previous trials or part of the trial itself can be used to learn a pattern (as shown in the last columns). 

\begin{table}
{
\footnotesize
\begin{center}
\begin{tabular}{p{1cm}p{3.2cm}p{3.2cm}p{3.2cm}p{3.2cm}p{0.2cm}}
 Task             & $K^{[\epsilon, \TAUC  \mapsto \infty]}(\mu)$       & $K^{[\epsilon, 
\mapsto  1]}(\mu)$           & $Kt^{[\epsilon, \TAUC  \mapsto \infty]}(\mu)$    & $Kt^{[\epsilon, \TAUC  \mapsto  1]}(\mu)$  & I \\ \hline
 $\mu_{num}$      & $\pi_{find-L-best}$ $\Longrightarrow$           & \natlangprog{Monte Carlo policy with many points}   & $\pi_{find-LS-best}$ $\Longrightarrow$      & \natlangprog{Monte Carlo policy with a few points} & M \\ 
 $\mu_{RPM}$      & $\pi_{find-L-best}$ $\Longrightarrow$           & \natlangprog{Shortest rpm solver}                     & $\pi_{find-LS-best}$ $\Longrightarrow$      & \natlangprog{LS-optimal rpm solver} & M \\ 
 $\mu_{Ctest}$    & \natlangprog{Monte Carlo search on sequence patterns} & \natlangprog{Monte Carlo search on sequence patterns} & \natlangprog{Levin search on sequence patterns}   & \natlangprog{Levin search on sequence patterns}   & -  \\ 
 $\mu_{response}$ & \natlangprog{react with minimum sleep periods}        & \natlangprog{react with minimum sleep periods}        & \natlangprog{react with fair sleep periods}       & \natlangprog{react with fair sleep periods} & -\\ 
 $\mu_{maze}$     & \natlangprog{right-hand traversal}                    & \natlangprog{right-hand traversal}                    & \natlangprog{LS-optimal traversal}                & \natlangprog{LS-optimal traversal} & -\\ 
 $\mu_{pacman}$   & $\pi_{find-L-best}$ $\Longrightarrow$           & \natlangprog{eat and escape from predators}           & $\pi_{find-LSbest}$ $\Longrightarrow$       & \natlangprog{eat and escape from predators} & M \\ 
 $\mu_{trans}$    & $\pi_{find-L-best}$ $\Longrightarrow$           & \natlangprog{shortest-translator}                     & $\pi_{find-LSbest}$ $\Longrightarrow$       & \natlangprog{LS-optimal-translator}  & M \\  
 $\mu_{imit}$     & \natlangprog{copy the observation to the action}      & \natlangprog{copy the observation to the action}      & \natlangprog{copy the observation to the action}  & \natlangprog{copy the observation to the action} & -\\  
 $\mu_{guess}$    & \natlangprog{guess randomly until reward}             & \natlangprog{guess randomly until reward}             & \natlangprog{guess randomly until reward}         & \natlangprog{guess randomly until reward} & -\\  
 $\mu_{eidetic}$  & \natlangprog{repeat what has been seen}               & \natlangprog{repeat what has been seen}               & \natlangprog{repeat what has been seen}           & \natlangprog{repeat what has been seen} & -\\  
 $\mu_{srote}$    & \natlangprog{output decompressible TEXT}              & \natlangprog{output decompressible TEXT}              & \natlangprog{efficiently decompressible TEXT}     & \natlangprog{efficiently decompressible TEXT} & -\\   
 $\mu_{lrote}$    & \natlangprog{copy text from previous trial}           & \natlangprog{output decompressible TEXT}              & \natlangprog{copy text from previous trial}       & \natlangprog{efficiently decompressible TEXT} & H \\   
 $\mu_{add}$      & \natlangprog{addition by incrementing}                & \natlangprog{addition by incrementing}                & \natlangprog{efficient addition}                  & \natlangprog{efficient addition} & -\\ 
 $\mu_{sudoku}$   & \natlangprog{exhaustive sudoku search}                & \natlangprog{exhaustive sudoku search}                & \natlangprog{efficient sudoku solver}             & \natlangprog{efficient sudoku solver} & -\\   \hline
\end{tabular}
\caption{Some of the illustrative tasks defined in Table \ref{tab:tasks} and the kind of policies that could lead to the minimisation of the complexity measures $K$ or $Kt$ with or without history.  
The cases with $\mapsto  1$ are blind to previous trials, either because there is not any previous trial or because the agent has no memory or is reinitialised for each trial. 
For those where $\pi_{find-L-best}$ or $\pi_{find-LS-best}$  appears, we assume there is no better policy (in terms of L or LS) that achieves $\epsilon$. The last column shows the few cases where there is a difference between many trials or just one trial. 
This effect of incrementality can be reflected in terms of algorithm self-improvement or meta-search, represented by M (and we also show a right arrow meaning that in the end it will be executing the algorithm on the right), and the use of history in other ways, H.}\label{tab:tasksopt}
\end{center}
}
\end{table}

The cases of $\mu_{srote}$ and  $\mu_{lrote}$ are significant. Both are just non-stochastic tasks that can be just done by rote-learning once a couple of instances are seen. In fact, this is an extreme case of stochastic tasks where there is a relevant part that is constant. 
We see that both are considered simple when there are several trials (either by memorising a short string or by using a policy that just memorises and copies it from the previous instance). 
The use of this copy\&paste policy can only be appreciated when the size of the thing to be copied has a certain value (for short strings, nothing can beat the policy with the string itself). This is a crucial example of why a blind search that tries to find policies without looking (and learning from) previous trials can be less efficient than another looking at previous trials. In other words, in an interactive scenario, an enumeration-like search might not be the best thing to do. This has been realised in some modifications of Levin's universal search for agents. 

} 

\commentAGI{
$Kt$ puts together the length of the policy and the computationl steps it takes, including both searching and execution. This makes it consider any meta-search procedure inside the policy, provided that this (with the information of the task) is more effective than getting the policy from nothing. In other words, if the task gives hints and there is a short and fast procedure that can use these hints to find the policy, then the exrpesion of $Kt$ will give this policy. 
Anyway, the use of $Kt$ is related to Levin's universal search \cite{Levin73,Li-Vitanyi08}, as if we measure the computational steps that are required to find the algorithm that minimises $Kt$ we have to go approximately through $2^{\Length(\pi)}$ programs with their corresponding execution steps of $\ESteps$. By multiplying these two terms and calculating a binary logarithm we have $Kt$. This connection, which will be better described later on, allows us to define the unit of difficulty as the logarithm of computational steps. 
}

\commentAGI{
We have said why $Kt^{[\epsilon, \TAUC  
\mapsto \infty]}(\mu)$ is not completely satisfying, since for some problems, the meta-search policy $\pi_{find-LS-best}$ is chosen. The reason is that, despite the great computationl effort of $\pi_{find-LS-best}$, this can concentrate on the first millions of trials and then progressively switching the behaviour so that the best policy so far is used. The problem is because we are calculating $Kt$ for an infinite number of trials. 
We have also discussed that $Kt^{[\epsilon, \TAUC  
\mapsto 1]}(\mu)$ cannot take history of previous trials. 
}

An option as an upper-bound measure of difficulty \commentAGI{in between} would be $\Hardness(\mu) \triangleq Kt^{[\epsilon, \TAUC  
\mapsto  \nuE]}(\mu)$, for a finite $\nuE$ and given $\epsilon$. That means that any search has to be done during evaluation and the computational steps here will be taken into account (if $\nuE$ is not too large). 
In general, if $\nuE$ is very large, then the last evaluations will prevail and any initial effort to find the policies and start applying them will not have enough weight. On the contrary, if $\nuE$ is small, then those policies that invest in analysing the environment will be penalised.   
%
That means that we will need to invest as little computation steps and trials to find an acceptable policy and then execute it for as many trials as needed to make $\EResponse\geq 1- \epsilon$. This is in a way a trade-off between exploration and exploitation.
It also requires a good assessment of the metasearch procedure to  {\em verify} the policy so it can go to exploitation.
In any case, the notion of difficulty depends, in some tasks, on $\nuE$. We will come back to this issue later on, as we will analyse the `verification cost', and how the number of trials $\nuE$ can be derived by a confidence degree such that the policy solving the problem is found and the trials can stop.

\section{Task instances, task composition and decomposition}\label{sec:instance}

Up to this point we have dealt with a first approach to {\em task} difficulty. A task includes (infinitely) many task instances. What about {\em instance} difficulty? Does it make sense? In case it does, instance difficulty would be very useful for adaptive tests, as we could make the stochastic task adaptive and start with simple instances and adapt their difficulty to the ability of the subject (as in adaptive testing in psychometrics). 

\commentAGI{
However, there are many confounding factors to determine the difficulty of a single instance.  For instance, for a division task  we may have these two instances: 6/3 and 1252/626. If the task is stochastic and includes many divisions, a policy that actually makes divisions will pay off. 
But if we create a task with just 6/3 or 1252/626 as only instances, in both cases the solution would be just 2, which is not only equal for both instances, but also a value that has no relation whatsoever to the difficulty of these instances.
}

The key issue is that instance difficulty must be defined {\em relative to a task}. 
At first sight, the difference in difficulty between 6/3 and 1252/626 is just a question of computational steps, as the latter usually requires more computational steps if a general division algorithm is used.
But what about 13528/13528? It looks an easy instance. Using a general division algorithm, it may be the case that it takes more computational steps than 1522/626. If we see it easy is because there are some shortcuts in {\em our} algorithm to make divisions. 
These shortcuts are frequently applied instead of the general procedure. 
 One of the shortcuts would be to return 1 if both arguments are equal. Of course, we can think about algorithms with many shortcuts, but then the notion of difficulty depends on how many shortcuts it has. In the end, this would make instance difficulty depend on a given algorithm for the task (and not the task itself). This would boil down to the steps taken by the algorithm, as in computational complexity. 
\commentAGI{For the relative numerousness task, for instance, the difficulty of an instance would be radically different if we are thinking about a counting policy (for which all instances are approximately equally easy) or we are thinking about a Monte Carlo policy (which depends on the difference in the total area of the circles, as the algorithm can stop when the difference is statistically significant).} 

We can of course take a structuralist approach, by linking the difficulty of an instance to a series of characteristics of the instance, such as its size, the similarities of their ingredients, etc. This is one of the usual approaches in psychology and many other areas, including evolutionary computation, but does not lead to a general view of what instance difficulty really is. For the divisions above, one can argue that 13528/13528 is more regular than 1252/626, and that is why the first is easier than the second. However, this is false in general, as $13528^{13528}$ is by no means easier than any other exponentiation. 

Some other approaches also link the difficulty of an instance or problem to the ``probability of failure'' \cite{bentley2004empirical} or to the ``probability-of-failure and mean time-to-solution''  \cite{ashlock2010evolution}. The probability of failure can be defined in terms of one policy (so we would have again a notion of difficulty dependent to the best policy solving the task), but another perspective is 
``the likelihood that a randomly chosen program will fail for any given input value'' \cite{bentley2004empirical}. This is interesting. Apparently, it looks like the population-based approach in psychology (apply the instance to some individuals and record times and success rates), as it is based on a population of programs. 

Here, we have several problems to follow this idea. We would need a population\footnote{We could assume a universal distribution of policies. This is related to the solution presented in this paper, since the shortest policies have a great part of the mass of this distribution.}. Also, we have that difficulty depends on computational cost and success rates, which are expressed in very different units.  
%
%
If the difficulty of a task is 8 (in logarithm of steps), what does it mean if we say that one of its instances has a difficulty of 0.3 (in proportion)? 
In any case, we may agree that computational cost and success rate are relevant, but they do not work in this way as a function of difficulty.


\commentAGI{
\subsection{Instance difficulty as rareness}
}

Instead of considering all policies\commentAGI{\footnote{As said above, we could also consider a universal distribution of policies, which would give a high probability to the best policy.}}, we can consider the best policy. The insight comes when we see that best policies may change with variable values of $\epsilon$. This leads to the view of the relative difficulty of an instance with respect to a task {\em as the minimum $\ELS$ for any possible tolerance of a policy such that the instance is accepted}. 

In order to formalise this concept, we must first formalise the notion of instance. For stochastic tasks, an instance is simply the very task for which its random behaviour is fixed. This can be obtained with the underlying model by setting a fixed string to the random tape or by setting a seed to the random generator (as in many computer languages). We denote by $\mu^\sigma$ an instance of $\mu$ by setting seed $\sigma$.

We first define the set of all optimal policies for varying tolerances $\epsilon_0$ as:
\begin{equation}\label{eq:argminepsi}
Opt_{\ELS}^{[\TAUC  \nuT \mapsto  \nuE]}(\mu) \triangleq \left\{ \argmin_{\pi \in \AccSet^{[\TAUC  \epsilon_0, \nuT \mapsto  \nuE]}(\mu)} {\ELS^{[\TAUC  \nuT \mapsto  \nuE]}(\pi,\mu)} \right\}_{\epsilon_0 \in [0,1]} 
\end{equation}
\noindent And now we define the instance difficulty of $\mu^{\sigma}$ with respect to $\mu$ as: 
\begin{equation}\label{eq:Ktinstance}
\Hardness^{[\epsilon, \TAUC  \nuT \mapsto  \nuE]}(\mu^{\sigma}|\mu) \triangleq \min_{\pi \in Opt_{\ELS}^{[\TAUC  \nuT \mapsto  \nuE]}(\mu) \cap  \AccSet^{[\TAUC  \epsilon, \nuT \mapsto  \nuE]}(\mu^{\sigma})} {\ELS^{[\TAUC  \nuT \mapsto  \nuE]}(\pi,\mu)} 
\end{equation}
\noindent The interpretation of the formulae above is as follows. Take all the optimal policies (in terms of $\ELS$) for varying values of $\epsilon$. Sort them by their $\epsilon$ increasingly. The first one that is acceptable for $\mu^{\sigma}$ gives {\em the best policy for $\mu$ that covers $\mu^{\sigma}$}. The $\ELS$ of this policy is the relative difficulty of $\mu^{\sigma}$ with respect to $\mu$. Note how the order of the minimisation is arranged in equations \ref{eq:argminepsi} and \ref{eq:Ktinstance} such that for the many policies that only cover $\mu^{\sigma}$ but do not solve many of the other instances, these are not considered because they are not in $Opt_{\ELS}$.

\commentAGI{
Let us see this for the relative numerousness task. Imagine the instance in Figure~\ref{fig:numerousness}. Let us choose a task tolerance $\epsilon = 0.1$, which we call the reference tolerance. Now consider all the possible policies solving the original task (considering many instances) when we vary the tolerance. For instance, for tolerances $\epsilon_0$ from $1$ to $0.5$ we have that the best policy is most likely one that always chooses left (or right), assuming that we have a balanced proportion of instances where the answer is left or right. For these tolerances, the error of this policy will always be acceptable. However, the error for $\mu^\sigma$ will be worse than the reference task tolerance level sets ($\epsilon=0.1$). The interest thing comes next, when we increase the tolerances. There might be a policy for tolerances $\epsilon_0$ 0.3 or 0.2 such that is also a policy for $\mu^{\sigma}$ with the reference task tolerance $\epsilon=0.1$. In this case, the $\ELS$ of this policy would be the difficulty of the instance. In other words, difficulty of an instance is the minimum effort for the whole task such that the instance is well covered. 
}

This notion of relative difficulty is basically a notion of consilience with the task. If we have an instance whose optimal policy is unrelated to the optimal policy for the rest, then this instance will not be covered until the tolerance becomes very low. Of course, this will depend on whether the algorithmic content of solving the instance can be accommodated into the general policy. This is closely related to concepts such as consilience, coherence and intensionality \cite{hernandez2000philosophica,hernandez2000explanatory,HernandezOrallo99a,HernandezOrallo00c,HernandezOrallo00d}. 
\commentAGI{If the instance is an outlier then it will be more difficult because it requires extra information to accommodate into the policy but also because the probability that it appears as part of the policy from the task is very low and hence it is hard that this case could be covered. In a way, difficulty is a notion of `rareness' ---in some senses of the term, special and unlikely.}

\commentAGI{
A different case is when there are many instances of some kind that are different from the rest of instances in the task. In this case, it is not an instance that is rare, but a set of instances, and it is better to analyse this in terms of task decomposition, as we will see in the following section. 
}

\commentAGI{
Finally, it must be said that equation~\ref{eq:Ktinstance} might be undefined for some instances, as none of the optimal policies for varying values of $\epsilon_0$ is able to cover it appropriately. This of course implies that in these cases there is no policy for the task with no tolerance ($\epsilon_0 = 0$). This is related to whether we define tolerance with respect to 1 or with respect to the best policy. In the latter case, the acceptable policy with no tolerance would always exist. But still, some instances might not be covered. That does not imply necessarily that there are no policies for these instances, but that there is no acceptable policy for these instances such that it is also an acceptable policy for all the other instances.
}





\commentAGI{
We can now see another example. 
For instance, in a task where the agent has to guess the following symbol in a letter series, such as the task $\mu_{Ctest}$ in Table~\ref{tab:tasks}  or Thurstone's letter series \cite{ThurstoneACE}, we may wonder why the series $aaaaaaa$ seems easier than $aacaeag$. 
Two explanations are here. First, given the previous definition, we can see that for high tolerance levels some simple policies may solve some series (e.g., a program that just solves arithmetic and geometric series would solve the first but not the last one). As a result, these simple incomplete policy would score some results if arithmetic and geometric series are a relevant proportion of all series. 
 This is exactly what the program passing IQ tests from \cite{sanghidowe2003computer} did, using some predefined rules for some common sequences. This would actually give a grading of instances, which some of them being in the same class (each class given by each of the policies returned by \ref{eq:argminepsi}). Second, we can of course assume the best policy overall (e.g., by considering the given tolerance $\epsilon$ or tolerance 0). The policy in this case, as shown in Table~\ref{tab:tasksopt}, would be a kind of Levin's search on the possible patterns. The difficulty would just be the computational steps of using this algorithm for the policy. As we mentioned above, there is a connection between the logarithm of the steps required by Levin's search and $Kt$, the measure of instance difficulty that was used in the $C$-test \cite{HernandezOrallo-MinayaCollado1998,HernandezOrallo2000a}.  
}

\commentAGI{Both explanations are sufficiently compelling to see whether both can be combined. 
A mixture of the two above approaches could be to modify equation~\ref{eq:Ktinstance} where $\Length$ is taken from the task policy while $\ESteps$ is taken for $\mu_{\sigma}$, i.e.:

\begin{equation}\label{eq:Ktinstance2}
Kt^{[\epsilon, \TAUC  \nuT \mapsto  \nuE]}(\mu_{\sigma}|\mu) \triangleq \min_{\pi \in Opt_{Kt}^{[\TAUC  \nuT \mapsto  \nuE]}(\mu) \cap  \AccSet^{[\TAUC  \epsilon, \nuT \mapsto  \nuE]}(\mu_{\sigma})} {\ELS^{[\TAUC  \nuT \mapsto  \nuE]}(\pi,\mu_{\sigma})} 
\end{equation}

This could be a more elaborate version of difficulty. 
Nonetheless, we must say that any of the above options is not proposed as a definitive policy that may give an intuitive value of difficulty for every possible task and instance. Our goal here is to show some of the ingredients about the notion of difficulty, and provide some useful references to construct one personalised version of instance difficulty for a given situation. Being more or less elaborate, we think that the principles must be the same. For instance, we emphasise that all of them are defined in terms of computational steps, so they actually measure algorithmic effort. 
}

Of course, there will be occasions where the notion of difficulty for an instance is controversial. 
For instance, for a division task, 
imagine 2525/8527, and now consider two different division algorithms $\pi_1$ and $\pi_2$, which both are approximately equally efficient and short (same $\ELS$). However, for $\pi_1$ we have that 7674/2558 is solved very easily but 7674/2558 takes much more steps for $\pi_2$. Using any of the above definitions, the difficulty will depend on the reference machine that calculates length and complexity. Other more robust options, such as considering not only the minimum in equations \ref{eq:argminepsi} and \ref{eq:Ktinstance}\commentAGI{ (or \ref{eq:Ktinstance2})}, have already been mentioned, and would lead to more complicated notions of difficulty. \commentAGI{In any case, they could be thought as future work.}


\commentAGI{
\subsection{Task composition and decomposition}\label{sec:composition} 
}

Now the question is to consider how we can put several tasks together. 
\commentAGI{For instance, if we include $\mu_{num}$ and $\mu_{RPM}$ from Table~\ref{tab:tasks} in the same test, does it make sense to aggregate the results?}  
The first problem is that the aggregation of several responses that are not commensurate makes no sense \commentAGI{(perhaps for one the responses go from 0 to 0.1 while for the other they go from 0.5 to 1, with very different distributions of results\footnote{One can normalise them by a cumulative distribution, again if we can figure out a population or distribution of policies.}).}  
One alternative to a normalisation is to use a tolerance level for the tasks.
 This gives further justification to eq. \ref{eq:acceptN}, where $\AccSet$ was introduced. Given two tolerance levels for each task we can see whether this leads to similar or different difficulties for each task. For instance, if the difficulties are very different, then the task will be dominated by the easy one. \commentAGI{In the previous example, the $\mu_{num}$ is much easier than the $\mu_{RPM}$. By using different tolerance levels we can determine whether we want both tasks to have the same relevance or not. In fact, we do not really need to use different values of $\epsilon$, as we can find a monotonic transformation of (one of) the responses such that the $\epsilon$ can be the same for both, leading to the same difficulty. Given any task, any monotonic transformation of the responses leads to another task such that there is another $\epsilon$ that leads to the same acceptability set.}  

Comparing the difficulties of the tasks for a response value is important to undertand what the composition really means, but we have not defined what a composition is. 
Given two stochastic tasks, it does not make sense to make the union of the tasks, but rather to calculate a mixture. In particular, the composition of tasks $\mu_1$ and $\mu_2$ with weight $\alpha \in [0,1]$, denoted by $\alpha \mu_1 \oplus (1-\alpha) \mu_2$, is defined by a stochastic choice, using a biased coin (e.g., using $\alpha$), between the two tasks. 
Note that this choice is made for each trial. It is easy to see that if both $\mu_1$ and $\mu_2$ are asyncronous-time stochastic tasks, this mixture also is. 

Similar to composition we can talk about decomposition, which is just understood in a straightforward way. Basically, $\mu$ is decomposable into $\mu_1$ and $\mu_2$ if there is an $\alpha$ and two tasks $\mu_1$ and $\mu_2$ such that $\mu = \alpha \mu_1 \oplus (1-\alpha) \mu_2$.


Now, it is interesting to have a short look at what happens with difficulty when two tasks are put together. Given a difficulty function $\Hardness$, we would like to see that if 
$\Hardness(\alpha \mu_1 \oplus (1-\alpha) \mu_2) \approx \alpha  \Hardness(\mu_1) +  (1-\alpha) \Hardness(\mu_2)$ then both tasks are related, and there is a common policy that takes advantage of some similarities. However, in order to make sense of this expression, we need to consider some values of $\alpha$ and fix a tolerance. With high tolerance the above will always be true as $\Hardness$ is close to zero independently of the task. 
With intermediate tolerances, if the difficulties are not even, the policies for the composed task will invest more resources for the easiest `subtask' and will neglect the most difficult `subtask'. For instance, if there is an easy policy for $\mu_1$ achieving response 0.8, but for $\mu_2$ the policies are much more difficult if the same level of response is aimed at, one can make do with a switch, use the easy policy for $\mu_1$ and manage with an easy policy for $\mu_2$ achieving response 0.4. If $\alpha=0.5$ then we would have overall response of 0.6, which may be acceptable for intermediate tolerances. 
Finally, using low tolerances (or even 0) for the above expressions may have more meaning, as the policy must take into account both tasks. In fact, for tolerance 0 the value of $\alpha$ that is not 0 or 1 is not relevant.

In fact, there are some cases for which some relations can be established. Assume 0 tolerance, and imagine that for every $1>\alpha>0$ we have $\Hardness(\alpha \mu_1 \oplus (1-\alpha) \mu_2) \approx \alpha  \Hardness(\mu_1)$. If this is the case, it means that we require the same effort to find a policy for both tasks than for one alone. We can see that task $\mu_1$ {\em covers} task $\mu_2$. In other words, the optimal policy for $\mu_1$ works for $\mu_2$. Note that this does not mean that every policy for $\mu_1$ works for $\mu_2$. 
Finally, if $\mu_1$ covers $\mu_2$ and vice versa, we can say that both tasks are equivalent.

We can also calculate a distance as $d(\mu_1,\mu_2) \triangleq 2\Hardness(0.5\mu_1 \oplus 0.5\mu_2) - \Hardness(\mu_1) - \Hardness(\mu_2)$. Clearly, if $\mu_1 = \mu_2$ then we have 0 distance. For tolerance 0 we also have that if $\mu_2$ has difficulty close to 0 but $\mu_1$ has a high difficulty $h_1$, and both tasks are unrelated but can be distinguished without effort, then we have that the distance is $h_1$. 

\commentAGI{
These ideas and properties can be related to concepts such as (normalised) information distance \cite{bennett1998information,vitanyi2009normalized}, especially the similarity of two tasks, with the appropriate caution, as here we are talking about interactive tasks and we are using the complexity of the policies and not the complexity of the description of the tasks. Two tasks with very similar (shortest) descriptions can have very different policies, and two tasks with very different (shortest) descriptions can have the same general policy. In the case of task composition and distance, we have seen that there are features for tasks that are fundamental, such as their magnitudes (which can be made even by the use of an appropriate tolerance or a monotonic function), their original difficulties and also whether both tasks can be distinguished easily (two similar tasks can be difficult to tell apart and putting them together could require a great extra effort).
}

Nonetheless, there are many questions we can analyse with this conceptualisation. For instance, how far can we decompose? That depends on how we decompose. For an infinite distribution (e.g., many stochastic tasks could be seen in this way), there are infinitely finer decompositions, each of them containing an infinite number of instances. But there are some other decompositions that will lead to tasks with very similar instances or even with just one instance. Let us consider the addition task $\mu_{add}$ with a soft geometrical distribution $p$ on the numbers to be added. With tolerance 0, the optimal policy is given by a short and efficient policy to addition. We can decompose addition into $\mu_{add1}$ and $\mu_{add2}$, where $\mu_{add1}$ contains all the summations $0+x$, and $\mu_{add2}$ incorporates all the rest. Given the distribution $p$, we can find the $\alpha$ such that $\mu_{add} =  \alpha \mu_{add1} \oplus (1-\alpha) \mu_{add2}$. From this decomposition, we see that $\mu_{add2}$ will have the same difficulty, as the removal of summations $0+x$ does not simplify the problem. However, $\mu_{add1}$ is simple now. But, interestingly, $\mu_{add2}$ still covers $\mu_{add1}$. We can figure out many decompositions, such as additions with and without carrying. Also, as the task gives more relevance to short additions because of the geometrical distribution, we may decompose the task in many one-instance tasks and a few general tasks. In the one-instance tasks we would put simple additions such as $1+5$ that we would just rote learn\commentAGI{ (the optimal policy for these cases alone is just rote learn)}. In fact, it is quite likely that in order to improve the efficiency of the general policy for $\mu_{add}$ the policy includes some tricks to treat some particular cases or easy subsets. 
\commentAGI{This can perfectly happen with some of the task difficulty functions seen before, such as $Kt^{[\epsilon, \TAUC  0 \mapsto  1]}(\mu)$. This is also consistent with many cognitive analyses of how humans perform addition (see, e.g., \cite{baroody2013development}). 
The use of decompositions can be useful to analyse many other cases. For instance, if we make a decomposition of $\pi$ into $\pi_1$ and $\pi_2$ with a high $\alpha$, and get that the difficulty of $\pi_1$ is low then it is quite likely that the original policy internally incorporates this separation. This can also happen with difficult instances (or subtasks) if tolerance is 0, as they can be incorporated in a rote-learning way. Also, it may be interesting to compare this (with tolerance 0) to the notion of instance difficulty seen in the previous section (which plays with levels of tolerance).
}


The opposite direction is if we think about how far we can reach by composing tasks. Again, we can compose tasks {\em ad eternum} without reaching more general tasks necessarily. The big question is whether we can analyse abilities with the use of compositions and difficulties. In other words, are there some tasks such that the acceptable policies for these tasks are frequently useful for many other tasks? That could be evaluated by looking what happens to a task $\mu_1$ with a given difficulty $h_1$ if it is composed with any other task $\mu_2$ of some task class. If the difficulty of the composed task remains constant (or increases very slightly), we can say that $\mu_1$ covers $\mu_2$. Are there tasks that cover many other tasks? This is actually what psychometrics and artificial intelligence are trying to unveil. For instance, in psychometrics, we can define a task $\mu_1$ with some selection of arithmetic operations and see that those who perform well on these operations have a good arithmetic ability. In our perspective, we would need to check whether the selection of operators that are evaluated ($+$, $-$, etc.) has some kind of optimal policy that does not help with the general problem. If this does not happen, then we could extrapolate (theoretically and not experimentally) that this task $\mu_1$ covers a range of arithmetic tasks.

\commentAGI{
As more general we get with composition, things will become harder (but not impossible). Can we define a task for inductive ability and show that this will cover every other pure inductive task? Or that it will be helpful for other tasks featuring inductive abilities? In artificial intelligence, this is usually the set of general techniques (in vision, pattern recognition, natural language processing) that are reused again and again in different applications. An ultimate question is whether there is a general task such that it is useful for every other task (like general intelligence or the $g$ factor), especially in cases with many trials (e.g., using $Kt^{[\epsilon, \TAUC  
\mapsto \infty]}(\mu)$). The view of Table~\ref{tab:tasksopt} and some of the optimal policies being $\pi_{find-LS-best}$ suggest that for very `large' tasks, these general, meta-search, algorithms may be good policies for these tasks.  
}

\commentAGI{
All of the above is just some directions that should be analysed in detail with separate research pieces.  
In practice, even if difficulty functions such as $Kt^{[\epsilon, \TAUC  
\mapsto  1]}(\mu)$ may be computationally expensive to calculate for many tasks, these set a conceptual framework to analyse many of these questions. For instance, the notion of task breadth can be analysed in this context.} 
There have been several (informal) approaches or expressions of relevance of task breadth \cite{goertzel2009toward,rohrer2010accelerating} or the notion of intellectual breadth (though applied to an agent \cite{goertzel2010toward} and not to a task). Some of them are relative to other tasks or to humans, such as the one suggested (but not fully developed) with the Turing Ratio \cite{masum2002turing}. 
With the observation of how difficulty changes with composition and decomposition we could try to give a a proper formalisation of task breadth or, alternatively, we may reach the conclusion that task breadth is not a meaningful notion. 

\commentAGI{
Finally, the additivity of difficulty with composition could be analysed, and compared to other kinds of combination. Imagine two tasks $\mu_1$ and $\mu_2$ that are put together (sequentially) as a new joint $\mu$, and an observation signals whether the agent is performing well for each of the two parts. A policy could be \natlangprog{Identify parts. Find the policy for the first part and the second part independently}. That means that the policy could do separate searches. For instance, in the worst case (or using a Levin search) we could have $2^{\Length(\pi_1)} + 2^{\Length(\pi_2)}$ possible choices instead of $2^{\Length(\pi_1)+\Length(\pi_2)}$, where $\pi_1$ and $\pi_2$ are the partial programs that solve the partial subtasks. A concatenation of tasks is very different from a composition, but if agents have memory, we could find cases where there are connections.
}

\comment{
A relevant question is whether the tolerance level and the rest of parameters should be the same for all tasks. 
We can group those of the same difficulty:
\begin{eqnarray}\label{eq:UAccN}
 \Psydiff{h}^{[\epsilon, \TAUC  \nuT \mapsto  \nuE]}(\pi,M,p_M) \triangleq  \sum_{\mu \in M, \Hardness^{[\TAUC  \epsilon, \nu]}(\mu)= h} p_M(\mu|h) \cdot \Acc^{[\TAUC  \epsilon, \nuT \mapsto  \nuE]}(\pi, \mu) 
\end{eqnarray}

It seems that if the tasks are different (such as $\mu_{num}$ and $\mu_{RPM}$), the parameters should be different. 
It makes sense then to define an environment setting as $\left\langle \mu, \epsilon, \TAUC  \nuT \mapsto  \nuE \right\rangle$ and redefine $M$ as the set of all settings. 
So we have:
\begin{eqnarray}\label{eq:UAccN2}
 \Psydiff{h}(\pi,M,p_M) \triangleq  \sum_{\left\langle \epsilon, \TAUC  \nuT \mapsto  \nuE \right\rangle \in M, \Hardness^{[\epsilon, \TAUC  \nuT \mapsto  \nuE]}(\mu)= h} p_M(\left\langle \mu, \nu, \epsilon \right\rangle|h) \cdot \Acc^{[\epsilon, \TAUC  \nuT \mapsto  \nuE]}(\pi, \mu) 
\end{eqnarray}
\noindent Note that whenever the parameters change (especially $\epsilon$) the measure of difficulty changes.

And then aggregating all difficulties:
%
%
\begin{eqnarray}\label{eq:psi2-disc2N}
\Psy_w(\pi,M,p_M) = \sum_{h=0}^{\infty} w(h) \Psydiff{h}(\pi,M,p_M)
\end{eqnarray}
}

\commentAGI{

\subsection{Agent response curves}

One of the usefulness of difficulty is the analysis of agents according to how they behave in terms of the difficulty of the problem. This can be done with the so-called agent response curves, introduced in \cite{upsychometrics2} following the notion of item response curves in psychometrics. Let us see briefly how these curves can be defined for tasks or task instances. 

We first define $\Acc^{[\TAUC  \epsilon, \nuT \mapsto  \nuE]}(\pi, \mu) \triangleq 1 \:\mbox{if}\: \pi \in \AccSet^{[\epsilon, \TAUC  \nuT \mapsto  \nuE]}(\mu) \:\mbox{and 0 otherwise}$. 
A task class is defined as a pair $\left\langle M, p_M \right\rangle$, where $M$ is a set of tasks or task instances, and $p_M$ is a distribution. Note that with this definition, task classes are stochastic tasks (but not all stochastic tasks can be seen as classes). We also consider a difficulty function $h$ (over tasks, or over task instances relative to the overall task). 

We can group those of the same difficulty:
\begin{eqnarray}\label{eq:psydiff}
 \Psydiff{h}^{[\epsilon, \TAUC  \nuT \mapsto  \nuE]}(\pi,M,p_M) \triangleq  \sum_{\mu \in M, \Hardness^{[\TAUC  \epsilon,  \nuT \mapsto  \nuE]}(\mu)= h} p_M(\mu|h) \cdot \Acc^{[\TAUC  \epsilon, \nuT \mapsto  \nuE]}(\pi, \mu) 
\end{eqnarray}

If we represent $\Psydiff{h}$ on the \yaxis versus $h$ on the \xaxis we have a so-called agent response curve, as shown in Figure~\ref{fig:arc}.
In order to have a nice view of the figure, we need to investigate how the points are derived. Do we have elements of $M$ with the same value of $h$? Otherwise, the values of the \yaxis would all be either 0 or 1. In order to observe values between them we must have several elements in $M$ with the same of $h$. If $h$ is a continuous function and this is not the case, we can group $h$ by intervals. This can also be done for convenience if there are no elements for some regions of $h$, so that we get a `continuous' curve without empty regions.

\begin{figure}[ht]
\vspace{-0.7cm}
\centering
\includegraphics[width=0.45\textwidth]{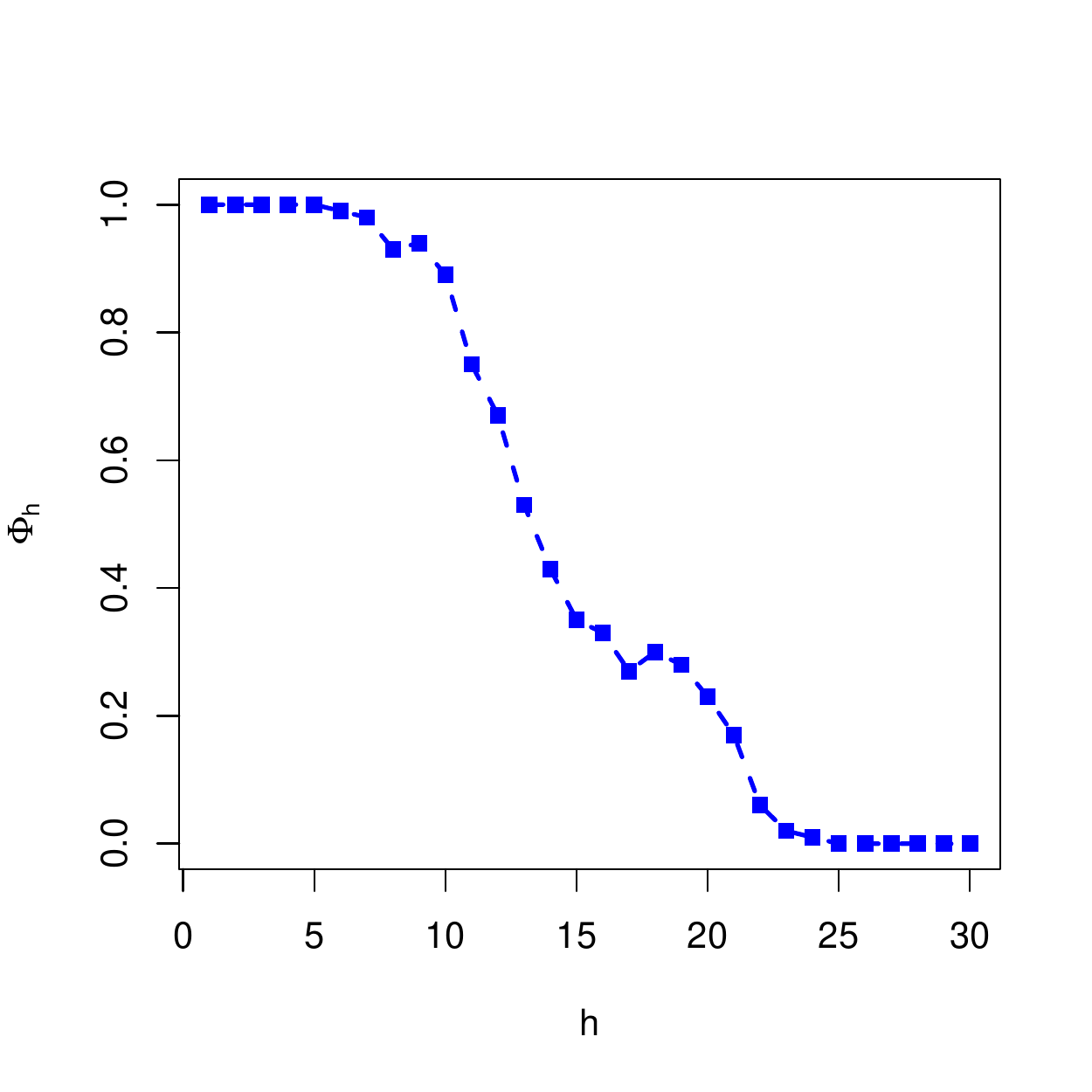} 
\vspace{-0.3cm}
\caption{An agent response curve.}
\label{fig:arc}
\end{figure}

The important thing, however, is how it works. Here we have three particular cases:

\begin{itemize}
\item If the elements of $M$ are stochastic tasks each and we use a non-relative version of difficulty, the curve may have a look very much the same as Figure~\ref{fig:arc}. For some of the difficulty functions seen in the previous sections ($K$ in particular\footnote{For $Kt$ this cannot be proved in general unless we include the computational steps the policy takes into $\Response$. However, this would go against a behavioural evaluation. Nonetheless, for those tasks for which there is some relevance of time to the $\Response$ and assuming non-infinite speed of the agent, we can show that this is bounded.}), it can be shown that for every agent there is an $h$ such that above it, $\Psydiff{h}$ is zero (this $h$ is actually the length of the algorithm of the agent). 
\item If the elements of $M$ are variants of the same task by varying the value of $\epsilon$ then we have that the curve is a non-increasing step function, where the leap of the step is located at the difficulty $h$ of the variant of the task such that the tolerance equals the achievement of agent $\pi$. This curve is of course not very informative for one agent. However, this could be interesting if we aggregate these `step' curves for a population of agents.
\item If the elements of $M$ are instances and we use a relative version of difficulty, assuming that the function is defined for all instances, we may have non-zero values for arbitrary high values of $h$, even for very simple agents (for instance, if an agent only solves a single instance, but this instance is difficult). At least this $h$ can be bounded by the $h$ of the task (if it exists for 0 tolerance). 
\end{itemize}

\noindent For the last case, we can see that if there are many instances with the same difficulty (or we aggregate values of $h$ in intervals), then we are considering an average of results for many instances and the shape will be mostly non-decreasing, like the one shown  in Figure~\ref{fig:arc}.



}

\section{Difficulty as Levin search with stochastic verification}\label{sec:verification} 


We decided to associate difficulty to the {\em smallest number of computational steps such that we get an acceptable policy for a given tolerance $\epsilon$}. 
This depends on how many alternative algorithms we need to try before we find the right one and how much time we require in order to discard the bad ones and confirm the correct one. This boils down to a measure of difficulty that depends on how many options need to be explored and the time that each of them takes. Their product will give an upper bound of the number of computational steps to find the best acceptable policy to a problem, i.e., its difficulty. 

In previous sections we considered the length of the policy and the logarithm of its computational time through their combination $\ELS$, which finally led to the function $Kt^{[\epsilon, \TAUC  \nuT \mapsto  \nuE]}(\mu)$. As we argued, this is given by the realisation that in order to find a policy of length $\Length(\pi)$ we have to try approximately $2^{\Length(\pi)}$ algorithms if we enumerate programs from small to large (this is basically what Levin search does, as we will see below). Considering that we can also gradually increase the computational steps that we devote for each of them, we get $2^{\Length(\pi)} \cdot \ESteps(\pi,\mu)$, whose logarithm is represented by $Kt$. This is why we say that the unit of $Kt$ is logarithm of computational steps\footnote{\cite{mayfield2007minimal} says ``this allows time to be measured in bits", but I think that this is misleading, as there is more information involved.}. 

If we try to extend this notion to tasks, the first, and perhaps most obvious and important difference with traditional Levin's universal search is that tasks are stochastic. Consequently, several trials may be needed for discarding a bad policy and the verification of a good one. This is specially the case when the response can have a high variance. Even a good policy can give bad results eventually, and we cannot discard a good policy just because it fails for one trial. We require repetitions, i.e., more trials, to know whether the policy is good for the whole task on average or not. 
Intuitively, a pair of task and policy with low variability in the response (results) will be easier to be verified than another where results behave more stochastically. \commentAGI{For instance, consider two possible policies: $a$ and $b$. Problem $\mu_1$ gets response 0 or 1 (with equal probability) for policy $a$ whereas it gives a constant value of 0 for policy $b$. This case is more difficult to find and verify than another problem $\mu_2$ where policy $a$ consistently gets a constant response of $0.5$. Note that in both cases, the expected response for policies $a$ and $b$ is 0.5 and 0 respectively, but it is intuitive to think that the first case is more difficult to find (actually because it is harder to verify).} 

\commentAGI{The second difference with classical Levin search is that the search algorithm goes through several trials, and it is not clear that the agent can interrupt the trial if a policy does not look promising. Nonetheless, we can consider that the search algorithm can also do some {\sf sleep} operations so that basically nothing is done until a new algorithm can be tried for the following trials.} 

The third difference is that we can think about a Levin search with memory, as some of the observations on previous trials may be crucial (whereas Levin search is basically a blind search). So we need that the policies that are tried could also be search procedures over several trials. That means that Levin search actually becomes  
 a metasearch, which considers all possible search procedures, ordered by size and resources\footnote{There are some variants and adaptations of Levin search for interactive scenarios and MDPs \cite{hutter2002fastest,schmidhuber2004optimal,Hutter05,schmidhuber2007godel,schaul2010metalearning}. 
Here it is not our goal to find a search that is useful to design intelligent agents but to find some expressions that help us refine our definition of {\em task} difficulty.}. \commentAGI{Only in this way we can properly give an intuitive measure of difficulty for $\mu_{srote}$ and  $\mu_{lrote}$ in Table~\ref{tab:tasksopt}.} 

\commentAGI{
This is just a realisation that for interactive stochastic scenarios, verification is not just one execution, 
%
but many if things are stochastic, because there is noise, the systems are not foolproof, etc. In a way, we are looking for more general and robust searches. This view is not very different to many evolutionary processes that have tried many policies in a world that is basically stochastic. 
} 


Another important thing is that in order to calculate the computational steps of a search, this search must stop at some point and say that the good policy has been found. However, as tasks are stochastic, we can never have complete certainty that a good policy has been found. 
An option is to consider a confidence level, such that the search invests as fewer computational steps as possible to have a degree of confidence $1-\delta$ of having found an $\epsilon$-acceptable policy. This clearly resembles a PAC (probably approximate correct) scenario \cite{valiant1984theory\commentAGI{,valiant2013probably}}.



\commentAGI{
Before starting, Table~\ref{tab:difftypes} summarises some of the notations we will use. We must also bear in mind that we are focussing on a view of difficulty when the policy is found by search (be it ``intellectual'', ``evolutionary'' or ``cultural'', as Turing distinguished \cite{turing1948intelligent}). However, the table also shows that there are other ways of acquiring a policy (by transmission, by demonstration or by search).

\begin{table}
{
\begin{center}
\begin{tabular}{cccc}
 Feature                             & Kind of difficulty it represents            & Notation                   & Depends on           \\ \hline
 Information content (size) of $\pi$ & Transmission (language or coding)     & $\Length(\pi)$             & -                    \\ 
 Execution steps of $\pi$            & Demonstration                         & $\ESteps(\pi,\mu)$          & -                    \\
 Expected value of response of $\pi$ & -                                                 & $\EResponse(\pi,\mu)$ & -                    \\
 True variance of response of $\pi$  & -                                                 & $\Var{\Response(\pi,\mu)}$ & -                    \\
 Verification trials of $\pi$        & -        & $\EBids(\pi,\mu)$           & $\EResponse$ and $\Var{\Response(\pi,\mu)}$    \\
 Finding effort steps of $\pi$              & Finding (trivial or no verification)                               & $\ELS(\pi,\mu)$             & $\Length$ and $\ESteps$ \\
 Verification steps of $\pi$         & -                                                 & $\EVerification(\pi,\mu)$   & $\ESteps$ and $\EBids$ \\
 Total effort steps of $\pi$     & Search (with target)               & $\EEffort(\pi,\mu)$         & $\Length$ and $\EVerification$ \\ \hline
\end{tabular}
\caption{Different features of a policy $\pi$ given a task $\mu$.}\label{tab:difftypes}
\end{center}
}
\end{table}

} 




\commentAGI{
\subsection{Levin universal search for stochastic tasks and/or policies}
}


Levin's universal search has very interesting properties, as any inversion problem can be solved optimally (except for a multiplicative constant) \cite[pp. 577--580]{Li-Vitanyi08}. It is related (and with approximately similar properties) to the SIMPLE search algorithm in \cite[pp. 579]{Li-Vitanyi08}, but with the advantage that the execution of programs does not need threads or traces to be kept in order to resume previously explored program executions (at the cost of repeating part of previous executions). \commentAGI{The important thing is how they relate the length of a program with their execution (and verification) time.}

\commentAGI{

The traditional Levin's universal search is defined as follows:


\begin{definition} Levin's universal search. \label{def:levinsearch1}
Given a string $x$ and a universal prefix-free Turing machine $U$, for which programs can be enumerated, we conduct several phases, starting from phase $1$.
For phase $i$, we execute all possible programs $p$ with $\Length(p) \leq  i$ for at most $s_i = 2^{i-\Length(p)}$ steps each, {\em including in this limit $s_i$ the steps\footnote{In order to verify a string we need to compare bit by bit with $x$. Note that this is not going to be constant. In the worst case, this takes $c \cdot \Length(x)$ steps, with $c$ being the computational steps per bit verification of a program that goes bit by bit over $x$. However, on average (assuming a 0.5 probability that a random program guesses each bit right), we have that the expected value is $\sum_{i=1}^{\Length(x)} i2^{-i}$, which converges to 2, so we will have $c \cdot 2$ steps on average. This is the reason why this verification part is often ignored for identification problems.} needed to verify whether $U(p) = x$.} 
 Once we find and verify the first successful policy the search is stopped. Otherwise we continue until we complete the phase and then to a next stage $i+1$.
\end{definition}


The number of steps to execute $p$ and verify that it produces $x$ or not\footnote{In this case, at the first moment that the string produced by $p$ does not match $x$ the verification is stopped.} is denoted by $\Verification(p,x)$.
We can determine an upper bound of the total number of steps taken by this procedure.  
While one could expect that this is $s\leq (k-1)2^{k+1} + 2$, this is significantly reduced by the use of prefix-free programs, so Kraft inequality can be used, having:

\comment{
\begin{theorem}
The number of steps $s$ taken by Levin search given by definition \ref{def:levinsearch1} is bounded by:
\[ s\leq (k-1)2^{k+1} + 2\]
where 
\begin{equation}
k = \Length(p) + \left\lceil log(\Steps(p))\right\rceil \label{eq:k}
\end{equation}
 and $p$ is the first program that meets the stop condition.
\end{theorem}
\begin{proof}
Let us see first that program $p$ is found in phase $k$. As it has size $\Length(p)$ amd requires $\Steps(p)$ steps, this is first achieved for this size when $\Steps(p) \leq 2^{i-\Length(p)}$. If we isolate $i$ from here we get
$i \geq \Length(p) + log(\Steps(p))$, which corresponds to the first natural number $k \geq i$, i.e., equation \ref{eq:k}.

For each phase $i$ we explore all possible programs $p$ with $\Length(p) \leq  i$, i.e., programs of length $j=1..i$ for at most $2^{i-j}$ steps. If all programs exhausted their step limit, that would make a total of steps per phase as follows:
\[ s_i = \sum_{j=1}^i 2^j \cdot \blue{\{} 2^{i-j} \blue{+ 2c \}}= \sum_{j=1}^i 2^i \blue{+ \sum_{j=1}^i 2^j 2c \}} = i2^i \blue{+ 2c (2^{i+1} - 2) \}}  \]
And the total number of steps is:
\[ \sum_{i=1}^k s_i = \sum_{i=1}^k i2^i \blue{+ \sum_{i=1}^k 2c (2^{i+1} - 2) } = (k-1)2^{k+1} + 2 \blue{+ 2c (2^{k+2} - 2k - 4)} \]
\end{proof}
}


\begin{theorem}\label{theo:nsteps}(\cite[pp. 580, claim 7.5.1]{Li-Vitanyi08})
The number of steps $s$ taken by Levin search given by definition \ref{def:levinsearch1} is bounded by:
\[ s\leq 2^{k+1} \]
where 
\begin{equation}
k = \Length(p) + log(\Verification(p,x)) \label{eq:k}
\end{equation}
 and $p$ is the first program that meets the stop condition.
\end{theorem}

\comment{
\red{
\[ \sum_{i:1 \leq i \leq k} \: \sum_{p:0 < i -l(p)} 2^{i-l(p)} = \sum_{i:1 \leq i \leq k} 2^i \sum_{p:0 < i -l(p)} 2^{-l(p)} \leq
\sum_{1 \leq i \leq k} 2^i \sum_{U(p) < \infty} 2^{-l(p)} \leq \sum_{U(p) < \infty} 2^{-l(p)} \sum_{1 \leq i \leq k} 2^i 
 \leq 2^{k+1} \]}
}

Even if we can use Kraft inequality and we get a much tighter upper bound, it seems that this bound is still rather loose, as many programs may stop before the alloted stops. 
However, as we can think of UTMs for which all programs of size lower than a constant may have the properties that we would like, this upper bound cannot be made lower in general (although some systems can exploit it for some other UTMs).


The use of $k$ as for equation \ref{eq:k} in theorem \ref{theo:nsteps} suggests that we use this expression as a standalone expression:
\begin{equation}
\log \Effort(p,x) \triangleq \log 2^{\Length(p)} \cdot \Verification(p,x) = \Length(p)+\log(\Verification(p,x))  \label{eq:effort}
\end{equation}
As we are considering non-probabilistic programs and an identification problem (and not really an inversion problem for any given partial recursive function), 
we do not need parameter $x$ here.
\[ \log \Effort(p) \triangleq \log 2^{\Length(p)} \cdot \Verification(p,U(p)) = \Length(p)+\log(\Verification(p,U(p)))  \]

According to the above process, it is easy to see that using the above procedure the first returned program that outputs $x$ will be  one that minimises:
$Kt(x) \triangleq \min_{p : U(p)=x} \log \Effort(p)$.

Levin's search assumes that there is a fast way of verifying policies.  
Now in the case of interactive stochastic systems with a response function, the procedure cannot just verify that the policy is correct by executing it once. Also, for each execution of the same program the number of steps can be different. How can we adapt universal search to this situation?

\begin{definition}\label{def:levinstochastic1} Levin's universal search for stochastic tasks and policies. 
Given a problem $x$ and a universal prefix-free Turing machine $U$, for which policies to $x$ can be enumerated, we conduct several phases, starting from phase $1$.
For phase $i$, we execute all possible programs $p$ with $\Length(p) \leq  i$ for at most $s_i = 2^{i-\Length(p)}$ steps each. {\em In these steps we include the steps required to execute $p$ several times to consider that $p$ is a policy for $x$} (within the alloted number of steps). As soon as the policy is deemed to be incorrect or the alloted number of steps is exhausted, we try the next program. On the contrary, if the policy is verified the search is stopped. 
While a policy is not found we continue until we complete the phase and then to a next stage $i+1$.
\end{definition}

The number of verification steps now\footnote{Actually, the number of verification steps was also an expected value, as depends on the differences between the reference string and the output of the program.} depend on stochastic executions and may vary (that is why we denote them by $\EVerification$). 
%
%
And similarly, we get the equation for effort equal to eq. \ref{eq:effort}. 
In this case, we cannot get rid of $x$ in the definition, as $p$ may be stochastic and $\EVerification$ is understood as an expected value.

\subsection{`PAC' verification for stochastic tasks}\label{sec:PACtarget} 


Now we are going to adapt this to stochastic tasks. We again realise that for a stochastic system we can never be 100\% sure of a policy, because even after a million successes we can have a failure. The second thing is that for stochastic systems it may be unreasonable to expect maximum or perfect result. In fact, by using any statistical test about whether we have reached the maximum value, we cannot have any degree of certainty (even the slimmest) of having this maximum value, as our average so far will never be above the maximum value, which is necessary for statistical significance.

We could think about using {\em one} slack parameter. Given a series of $n$ runs, we can calculate the average $\widehat{r}$ and the standard deviation $\sigma$ of the results. For instance, we can get the standard error by just $\mbox{SE} = \frac{\sigma}{\sqrt{n}}$ and set a limit on it. However, if we do this, we see that it depends on the magnitude. For instance, a stochastic process alternating between 0 and 1 will have higher $\sigma$ than if it alternates between 0.4 and 0.6, just by scaling, even if the verification cost (of knowing whether it is above, e.g., 0.7, or not) looks the same. 


Instead, we are going to consider two parameters. 
} 
 We want the search procedure to find a policy with a confidence level $\delta$, i.e., $Pr(\pi \: \mbox{solves} \: \mu) \geq 1 - \delta$. As mentioned above, if we consider the best possible result (i.e., 1) to acknowledge that this is solved, then even with high values of $\delta$ this will never be achieved. So the second thing is that if we consider a utility, response or result function $\Response$, we must set that the difference with respect to the best policy is lower than a given error $\epsilon$. 
If we denote the best possible average result (for an infinite number of runs) as $r^*$ (note that $r^*$ can be lower than 1), we consider that a series of runs is a sufficient verification for a probably approximate correct (PAC) policy $\pi$ for $\mu$ when:
%
\begin{equation}
Pr(r^* - \widehat{r} \leq \epsilon) \geq 1 - \delta \label{eq:pac}
\end{equation}
\noindent with $\widehat{r}$ being the average of the results of the trials (runs) so far. 
\commentAGI{As $r^* - \widehat{r} \leq \epsilon$ is the same as $\widehat{r} \geq r^* - \epsilon$, sometimes $r^* - \epsilon$ will be referred to as the `threshold' or `target'. For instance, if the achievable maximum is 0.9 and $\epsilon = 0.15$ then our threshold is 0.75.}



\commentAGI{
Now we are ready to give an expression for the verification steps for a given problem $\mu$ and a policy $\pi$. Namely, 
the number of verification steps $\EVerification^{[\TAUC\epsilon, \delta]}(\pi, \mu)$ is defined as the expected value of the parameter $s$ returned by \textsc{VerifyGen} (Algorithm~\ref{alg:algo1}) for $s_{max} = \infty$.


\begin{algorithm}
\caption{\textsc{Verification algorithm (generic)}} \label{alg:algo1}
\
\begin{algorithmic}[1]
\Function{VerifyGen}{$\pi$, $\mu$, $\TAUC$ $\epsilon$, $\delta$, $s_{max}$}  \Comment{$s_{max}$ is the number of allowed steps}
\State $j \leftarrow 1$
\State $s \leftarrow 0$
\State $m_\pi \leftarrow \emptyset$                             \Comment{The algorithm $\pi$ can keep memory between trials. Initially empty.}
\Repeat
\State $\left\langle r_j,s_j,m_\pi \right\rangle \leftarrow \Run(\pi, m_\pi, \mu, \TAUC  s_{max}-s)$  \Comment{One trial with at most the $s_{max}-s$ remaining steps}
\State \Comment{$\Run$ returns response and used steps}
\State $s \leftarrow s + s_j$                                            \Comment{Accumulate steps}
\State $r \leftarrow r + r_j$                                            \Comment{Accumulate response}
\State $\widehat{r} \leftarrow \frac{r}{j}$                                  \Comment{Average response}
\State $p \leftarrow Pr(r^* - \widehat{r} \leq \epsilon)$                    \Comment{We calculate this probability in some way}
\If{$p \geq 1 - \delta$} \Return $\left\langle \mbox{\sf TRUE},s \right\rangle$ 
\Comment{Stop because it is verified}
\ElsIf{$p \leq \delta$} \Return $\left\langle \mbox{\sf FALSE},s \right\rangle$ \Comment{Stop because it is rejected}
\EndIf 
\State $j \leftarrow j + 1$
\Until{$s \geq s_{max}$}
\State \Return $\left\langle \mbox{\sf FALSE},s \right\rangle$
\EndFunction
\end{algorithmic}
\end{algorithm}

Note that we require $r^*$, which is defined as the highest expected response of any resource-bounded policy (in $\ELS$). If this is not known, we can assume $r^*=1$, as in previous sections.

Algorithm~\ref{alg:algo1}, if using an appropriate estimation of the probability for stopping in each iteration, may have a tendency of stopping prematurely because each iteration depends on the previous ones. Actually, especially in the beginning, this is vulnerable to spurious results and very bad estimations of the mean and the variance of the response. 
This is basically the problem of (large-scale) multiple testing. 
One strong correction is Bonferroni method, where the confidence per test is modified to: $\delta' = \delta / n$, but as $n$ are the repetitions so far, it would be an incremental test. 
\comment{On the other extreme, we have people that says that no adjustments are needed (Rothman 1990): ``scientists should not be so reluctant to explore leads that may turn out to be wrong that they penalize themselves for missing possibly important''.  
But others criticise this view: Bender, and Lange 1999 ``Adjusting for multiple testing: when and how''. 
}

Also, if used with definition \ref{def:levinstochastic1}, this problem is exacerbated. As we evaluate about $2^i$ programs in each phase (actually slightly less than this because it is a prefix code), we have that there can be cases that are just accepted (the condition $Pr(r^* - \widehat{r} \leq \epsilon) \geq 1 - \delta$ becomes true) by chance. Note that a rejection by chance is not so problematic, as the same program will be evaluated in the following phase again. 


Given all the considerations above, we realise that it is going to be very difficult to find the exact statistical criteria to stop by acceptance and rejection. In what follows, we just propose an approximation to the upper limit with the goal of recognising the difference in difficulty between finding an acceptable policy for stochastic problems with high margin ($\widehat{r} + \epsilon - r^*$) and small standard deviation and those with tighter margins and higher standard deviations.

} 

First, we are going to assume that all runs take the same number of steps (a strong assumption, but let us remind that this is an upper limit\commentAGI{\footnote{For instance, if $\epsilon =0.5$ and all bad policies have response 0.499999 and there is only one good policy with response 0.55, we will require many repetitions to discard the bad policies, until we find and verify the good policy.}}), so the verification cost \commentAGI{above} could be approximated by 
\begin{equation}
\widehat{\EVerification}^{[\TAUC\epsilon, \delta]}(\pi, \mu) \triangleq \ESteps(\TAUC \pi,\mu) \cdot \EBids^{[\TAUC\epsilon, \delta]}(\pi, \mu) \label{eq:VerificationApprox}
\end{equation}
i.e., the product of the expected number of steps times the expected number of verification bids\commentAGI{ (iterations needed of the loop of Algorithm~\ref{alg:algo1})}. 
With this, we focus on calculating the number of bids of the policy until we verify it is a acceptable or not.

The number of bids can be estimated if we have the mean and the standard deviation of the response for a series of runs. 
\commentAGI{If the conditions of the central limit theorem held, we could consider that the results of the bids would be normally distributed. In our case, the trials are not independent (neither are they ergodic) if we consider that the algorithm has memory between the trials, but nevertheless we will make this assumption, as, in general, we cannot make any further assumption about the distribution of the responses. As a result we can use the confidence level given by the normal distribution.
The confidence interval is given by $\widehat{r} - \frac{|z_{\delta/2}|\sigma}{\sqrt{n}},  \widehat{r} + \frac{|z_{\delta/2}|\sigma}{\sqrt{n}}$.
Where $z_{\delta/2}$ is the standard normal quantile. For instance, for $\delta=0.05$, we have $|z_{0.025}|=1.96$. 
We want this interval width $w$ to be at most twice the margin over the threshold $r^* - \epsilon - \widehat{r}$. 
So, $w \leq 2(\widehat{r} + \epsilon - r^*)$. As $w = 2\frac{|z_{\delta/2}|\sigma}{\sqrt{n}}$, we have:
$2\frac{|z_{\delta/2}|\sigma}{\sqrt{n}} \leq 2(\widehat{r} + \epsilon - r^*)$.
By isolating $n$ we have:}
\commentARXIV{Assuming a normal distribution:} 

\begin{equation}\label{eq:znorm}
n \geq \frac{|z_{\delta/2}|^2\sigma^2}{(\widehat{r} + \epsilon - r^*)^2}
\end{equation}
  
\commentAGI{
Note that the above formula is infinite when $r^* - \epsilon = \widehat{r}$, i.e., when we have that the policy reaches the threshold exactly. We cannot verify it is above the threshold for any confidence level.}   
%
%

In order to apply the above expression we need the variance $\sigma^2$. If we just have one run, this is undefined, and for very few runs this is going to be poorly estimated. 
Many approaches to the estimation of a population mean with unknown $\sigma^2$ are based on a pilot or prior study (let us say we try 30 repetitions) and then derive $n$ using the normal distribution and then use this for a Student's t distribution. 
Instead of this, we are going to take an iterative approach where we update the mean and standard deviation after each repetition. The problem, of course, happens with the first iterations. One approach we will take is to consider the maximum standard deviation as a start (as a kind of Laplace correction). As we assume that the response $\Response$ is between 0 and 1, we will consider\footnote{Other options exist, such as deriving some initial values depending on the threshold.} 
 two fabricated repetitions with responses 0 and 1. 
\commentAGI{With this, our start sample standard deviation will be high from the beginning and a minimum of iterations will always take place.}


\commentAGI{Algorithm~\ref{alg:algo2} is a modification of Algorithm~\ref{alg:algo1} where we use eq.~\ref{eq:znorm}.} 
\commentARXIV{Algorithm~\ref{alg:algo2} is used in a modified Levin search:}

\begin{algorithm}
\caption{\textsc{Verification algorithm (normality)}} \label{alg:algo2}
\
\begin{algorithmic}[1]
\Function{VerifyNorm}{$\pi$, $\mu$,$\TAUC$$\epsilon$, $\delta$, $s_{max}$}  \Comment{$s_{max}$ is the number of allowed steps}
\State $j \leftarrow 3$                                           \Comment{We consider two first response with high variance}
\State $r \leftarrow 0+1$                                         \Comment{One with value 0 and the other with value 1} 
\State $s \leftarrow 0$
\State $m_\pi \leftarrow \emptyset$                             \Comment{The algorithm $\pi$ can keep memory between trials. Initially empty.}
\Repeat
\State $\left\langle r_j,s_j,m_\pi \right\rangle \leftarrow \Run(\pi, m_\pi, \mu, \TAUC  s_{max}-s)$  \Comment{One trial with at most}
\State \Comment{the $s_{max}-s$ remaining steps $\Run$ returns response and used steps}
\State $s \leftarrow s + s_j$                                            \Comment{Accumulate steps}
\State $r \leftarrow r + r_j$                                            \Comment{Accumulate response}
\State $\widehat{r} \leftarrow \frac{r}{j}$                                  \Comment{Average response}
\State $\widehat{\sigma}^2 \leftarrow \Var{r_1 \dots r_j}$                                  \Comment{Variance estimation}
\State $n_0 \leftarrow \frac{|z_{\delta/2}|^2\widehat{\sigma}^2}{(\widehat{r} + \epsilon - r^*)^2}$
\If{$j \geq n_0$} \If{$\widehat{r} > r^* - \epsilon$} \Return $\left\langle \mbox{\sf TRUE},s \right\rangle$ 
\Comment{Stop because it is verified}
\Else $\:$ \Return $\left\langle \mbox{\sf FALSE},s \right\rangle$ \Comment{Stop because it is rejected}
\EndIf   
\EndIf
\State $j \leftarrow j + 1$
\Until{$s \geq s_{max}$}
\State \Return $\left\langle \mbox{\sf FALSE},s \right\rangle$
\EndFunction
\end{algorithmic}
\end{algorithm}

\commentAGI{Finally, we modify definition \ref{def:levinstochastic1} by considering that when we find a verified policy we repeat the verification again with some extra repetitions (for instance, $n=30$, so that the used normal distribution is a more sustainable assumption). Note that this extra verification will be performed just very occasionally, so this will not significantly affect the number of steps taken by the modified Levin search. With all this, the modified version is as follows:}

\begin{definition}\label{def:levinstochastic2} Levin's universal search for stochastic tasks and policies with given tolerance $\epsilon$,  confidence level $1-\delta$, and maximum response reference $r^*$. 
Given a task $\mu$ 
for which policies can be enumerated. 
We conduct several phases, starting from phase $1$.
For phase $i$, we execute all possible policies $\pi$ with $\Length(\pi) \leq  i$ for $s_{i} = 2^{i-\Length(\pi)}$ steps each. 
We call function \textsc{VerifyNorm}$(\pi, \mu, \TAUC \epsilon, \delta, s_{max})$ 
in Algorithm~\ref{alg:algo2} with $s_{max} = s_i$. 
While an acceptable policy is not found we continue until we complete the phase and then to a next stage $i+1$. 
If an acceptable policy is found, some extra trials are performed before stopping the search for confirmation.
\end{definition}

\begin{theorem}
For every $\mu$ and $\epsilon, \delta > 0$, if a maximum $r^*$ exists achievable by a computable policy\footnote{It could not exist if there is a never-ending series of programs requiring, e.g., more time to get a slightly better policy. It exists if there is a limit of steps (not time) with the interaction with the environment.} 
and it is given, then definition \ref{def:levinstochastic2} conducts a finite search.
\end{theorem}
\begin{proof}
As $r^*$ is defined as the highest expected response for a resource-bounded policy $\pi^*$ (in $\ELength$) and it exists, then there is a number of phases where $\pi^*$ has already been found and there are enough steps such that $\widehat{r}$ is becoming as closer to $r^*$ as needed such that $\widehat{r} + \epsilon - r^*$ is positive and sufficiently close to $\epsilon$ such that is verified $Pr(r^* - \widehat{r} \leq \epsilon) \geq 1 - \delta$. Note that as results are bounded between 0 and 1 the highest variability is $\sigma^2 = 1/4$, so we have that $n \sim \frac{|z_{\delta/2}|^2\sigma^2}{(\epsilon)^2}$ is bounded.
\end{proof}

\commentAGI{
Note that if instead of $r^*$ we give a higher value that, subtracted the error tolerance, cannot be attained, then the search is not bounded. Also note that in any case there can be a very simple policy equal to $r^* - \epsilon$ and will never be found.

Definition \ref{def:levinstochastic2} is conceived to find the optimal policy, and it is not parametrised to calculate how long the search is to {\em discard} non-optimal policies. 
Actually, what we do is to use the approximation (i.e., equation \ref{eq:VerificationApprox}) into another approximation for any possible $\pi$, assuming that $\pi$ were the best policy.
}

In the end, what we want is to have a term that accounts for the variability of computational steps given by the variance of the response and its proximity to the threshold, as both things make verification more difficult. This is finally calculated as:
\begin{equation}\label{eq:ebids}
\EBids^{[\TAUC\epsilon, \delta]}(\pi, \mu) \triangleq \frac{|z_{\delta/2}|^2 {\Var{\Response(\pi,\mu)}}}{(\EResponse(\pi, \mu) + \epsilon - r^*)^2}
\end{equation}
\noindent For both $\Var{\Response(\pi,\mu)}$ and $\EResponse(\pi, \mu)$ we consider that we include two extra responses as a start, as done in Algorithm~\ref{alg:algo2}.

And now the effort \commentAGI{(eq. \ref{eq:effort})} is rewritten as:
\begin{equation}
\log \EEffort^{[\TAUC\epsilon, \delta]}(\pi, \mu) \triangleq \log (2^{\Length(\pi)} \cdot \widehat{\EVerification}^{[\TAUC\epsilon, \delta]}(\pi, \mu)) = \Length(\pi) + \log \widehat{\EVerification}^{[\TAUC\epsilon, \delta]}(\pi, \mu) 
\end{equation}
\noindent For clarity, we can expand what $\EEffort$ is by using the definition $\widehat{\EVerification}$ from eq.~\ref{eq:VerificationApprox} and taking the bids from eq.~\ref{eq:ebids} as:
\begin{equation}\label{eq:finaleffort}
\log \EEffort^{[\TAUC\epsilon, \delta]}(\pi, \mu) = \Length(\pi) + \log \ESteps(\TAUC \pi,\mu) \cdot \EBids^{[\TAUC\epsilon, \delta]}(\pi, \mu) 
=  \Length(\pi) +  \log \ESteps(\TAUC \pi,\mu) + \log \EBids^{[\TAUC\epsilon, \delta]}(\pi, \mu) 
\end{equation}
%
\noindent\commentAGI{It is a good question to determine whether $\ESteps$ and $\EBids$ have comparable magnitudes. If the policies take thousands of steps and the number of repetitions is in the order of dozens or hundreds, then the variability of the responses will not be very important, and the difficulty will be dominated by $\Length$ and $\ESteps$. However, this depends on the task; there are of course cases for which $\EBids$ can be very relevant.

}From here, we can finally define a measure of difficulty that accounts for all the issues that affect the search of the policy for a stochastic task:
\begin{equation}\label{eq:finaldifficulty}
\Hardness^{[\TAUC\epsilon, \delta]}(\mu) \triangleq min_\pi \log \EEffort^{[\TAUC\epsilon, \delta]}(\pi, \mu)
\end{equation}
\commentAGI{
\noindent It is important to compare this definition with those in section \ref{sec:KandKt}\commentAGI{ and Table~\ref{tab:tasks}}. 
Algorithm~\ref{alg:algo2} considers that the algorithm has memory between tasks, so we are really extending $Kt^{[\epsilon, \TAUC  
\mapsto  \nuE]}(\mu)$ in section \ref{sec:KandKt} ---but it can be modified easily without memory. \commentAGI{The good thing is that now we do not need to specify $\nuE$ any more, as the number of trials is given by Levin's search itself. This takes some of the (best) cases from the two columns of Table~\ref{tab:tasks} with $Kt$.} 
}

\commentAGI{

\subsection{Interpretation and use}


\comment{
\subsection{PAC verification refined (without target)}\label{sec:PACnotarget}

The previous search procedure and associated measures rely on Levin search knowing the threshold or target (i.e., knowing $r^* - \epsilon$). 
We would like to consider another kind of search where there is no {\em target} and the search is just looking for better and better policies without any particular target.

What we are going to do is considering that $\epsilon$ is linked to the confidence interval of a good estimation of a result for every policy, but not linked to any particular target.

\begin{definition}\label{def:levinstochastic3} Levin (universal) search for stochastic tasks or policies without target. 
Given a problem $x$ and a universal prefix-free Turing machine $U$, for which policies to $x$ can be enumerated. 
We conduct several phases, starting from phase $1$. 
For each phase $i$, 
we {\em reset} \red{(NOT CLEAR)} the best policy so far, denoted by $\pi^o_i$ and its estimated result as $\widehat{r}^o_i$. 
We execute all possible programs $p$ with $\Length(p) \leq  i$ for $2^{i-\Length(p)}$ steps. {\em In these steps we include the steps required to execute $p$ several times to consider that $p$ is an acceptable policy for $x$} (within the alloted number of steps). 
Now we check $Pr(r - \widehat{r} \leq \epsilon) \geq 1 - \delta$, i.e., that the estimation of the result of the program is sufficiently good.
A policy is considered verified when the $1 - \delta$ confidence interval is smaller than $2\epsilon$, which leads to $n \geq \frac{|z_{\delta/2}|^2\sigma^2}{(\epsilon)^2}$. 
If it is verified, we compare $\widehat{r}$ with $\widehat{r}^o_i$, and if it is better, we update $\pi^o$. 
As soon as the policy is deemed to be verified or the alloted number of steps is exhausted, we try the next program.  
When the phase is finished, we consider that $\pi^o_i$ is the best policy for this phase $i$. It is output and forgotten \red{(NOT CLEAR. WE CAN KEEP ITS RESULTS OR AVERAGE AND VARIANCE AND ACCUMULATE EVIDENCE.)}. 
\end{definition}

From the above procedure, we can now (externally) set a target $r^* - \epsilon$ and estimate the effort as the first phase $i$ where a policy $\pi^o_i$  can be output such that $Pr(r^* - r^o_i \leq \epsilon) \geq 1 - \delta$ (i.e., the probability of being output in $i$ is higher than $1 - \delta$).

As the number of hypotheses is about $2^{i}$ for each phase \red{(much lower than this if prefix code)}, it is likely that a spurious policy can get a small confidence interval with a high result (by chance, i.e., several good results in a row). The problem is not that a spurious policy can be output (because this will be different each time, but that the good one is not output). And this likelihood can increase with $i$? Perhaps not if we accumulate results of the best policy and we repeat the challenger until the interval is equally wide.
We can add extra checks (repetitions) whenever the comparison between $\widehat{r}$ and $\widehat{r}^o$ is positive. 
These extra checks COULD DEPEND ON THE PHASE, which would depend on $\Length$ and $\Steps$.

Assuming this issue is resolved, if there is a $\pi$ such that $r > r^* - \epsilon$, then this will be found at most with $n \sim \frac{|z_{\delta/2}|^2(1/2)^2}{(\epsilon)^2}$ repetitions. So $n$ is bounded by a value that is independent of the policy or the task (only depends on $\epsilon$ and $\delta$, and perhaps on $\Length$ and $\Steps$ if we add these extra checks depending on the phase). To see this it is enough to consider that as results are bounded between 0 and 1 the highest variability is $\sigma^2 = 1/4$, so we have that $n \sim \frac{|z_{\delta/2}|^2\sigma^2}{(\epsilon)^2}$ is bounded.
}



Does the approximation in equations \ref{eq:finaleffort} and \ref{eq:finaldifficulty} work properly? In order to get more insight about how it works we are going to see some figurative examples and see the values that would result, in order to see the effect of $\EBids$ in the new formula of difficulty.
This is shown in Table~\ref{tab:examples} (all cases consider $r^* = 1$). 
As we see from the results, there are cases where $\EBids$ can be large and have effect on $\log(\EEffort)$.

\begin{table}
{
\small
\begin{center}
\begin{tabular}{ccccccccccc}
 responses       & $1-\epsilon$ & $1-\delta$ & $\Length(\pi^*)$ & $\ESteps(\pi^*,\mu)$ & $\EResponse$ & $\sigma$ & $\EBids \:\mbox{with}\: \widehat{\sigma}$    & $\EBids \:\mbox{with}\: \sigma$ & $\log(\EEffort) \:\mbox{with}\: \widehat{\sigma} $ \\ \hline
 anything        & anything     & 1          &                  &                      & -            & -        & $\infty$     &   $\infty$        &           \\    
 1               & 1            & anything   &                  &                      & 1            & 0        & 49           & NaN               &           \\    
 1               & 0.99         & 0.95       &                  &                      & 1            & 0        & 197          & 0    &           \\    
 0*              & 0.01         & 0.95       &                  &                      & 0            & 0        & 197          & 0    &           \\    
 1               & 0.3          & 0.95       &                  &                      & 1            & 0        & 2            & 0    &           \\    
 0.3             & 0.3          & 0.95       &                  &                      & 0.3          & 0        & $\infty$     & 0    &           \\    
 $\{0.002,0.018\}$ & 0.01         & 0.95       &                  &                      & 0.019        & 0.008    & $\infty$ & $\infty$ &           \\    
 $\{0.002,0.018\}$ & 0.009        & 0.95       &                  &                      & 0.019        & 0.008    & 28       & 62 &           \\     
 1 one in 100    & 0.009        & 0.95       &                  &                      & 0.01         & 0.0995   & 7600     & 9507 &           \\     
 N(0.01,0.001)   & 0.009        & 0.95       &                  &                      & 0.01         & 0.001    & 26       & 9  &           \\     
 0.01            & 0.009        & 0.95       &                  &                      & 0.01         & 0        & 649      & 0    &           \\    
 $\{0,1\}$       & 0.45         & 0.95       &                  &                      & 0.5          & 0.5      & 98       & 97 &           \\    
 0.5             & 0.45         & 0.95       &                  &                      & 0.5          & 0        & 14       & 0    &           \\    
 0.55            & 0.5          & 0.95       &                  &                      & 0.55         & 0        & 16       & 0    &           \\    
 0.45*           & 0.5          & 0.95       &                  &                      & 0.45         & 0        & 16      & 0     &           \\    
 0.5             & 0.3          & 0.95       &                  &                      & 0.5          & 0        & 4        & 0    &                 \\    
 0.5             & 0.3          & 0.9        &                  &                      & 0.5          & 0        & 3        & 0    &                 \\    
 0.5             & 0.3          & 0.99       &                  &                      & 0.5          & 0        & 5        & 0    &                 \\
 1               & 0.5          & 0.95       & 10               & 200                  & 1            & 0        & 4        & 0    & 10 + log(800) = 19.6 \\    
 0.51            & 0.5          & 0.95       & 5                & 100                  & 0.51         & 0        & 93       & 0    & 5 + log(9300) = 18.2 \\    
 0.51            & 0.5          & 0.95       & 7                & 20                   & 0.51         & 0        & 93       & 0   & 7 + log(1860) = 17.0 \\    
 0.51            & 0.5          & 0.95       & 10               & 200                  & 0.51         & 0        & 93       & 0   & 10 + log(18600) = 24.2 \\    
 anything        & 0            & anything   & -                & -                    & -            & -        & 0        & 0   &    \\    
\hline
\end{tabular}
\caption{Examples of stochastic tasks. We are assuming $r^*=1$ (see second column). We figure out an optimal policy $\pi^*$ and see what value for $\EEffort$ would result (note that we are not doing an actual Levin search here). All estimations are using smoothing by the inclusion of a result 0 and 1 at the beginning of the results vector, as in Algorithm~\ref{alg:algo2}. We use $1,000$ trials to calculate the true expected response and the true variance. $\EResponse$ and $\sigma$ are shown without the smoothing. The expected number of bids ($\EBids$ with $\widehat{\sigma}$) is calculated incrementally until the number of repetitions needed to calculate a value of $n$ (the repetitions) is lower than the current iteration, as if the variance were approximated incrementally. The same calculation with the perfect value of $\sigma$ is represented in the next column: ($\EBids$ with $\sigma$). Note that $\EVerification$ is approximated here as a product of expected values, as it is actually the expected value of an algorithm using many runs.  The asterisks in the first column represent that these are cases played with policies that are rejected (just for comparison).}\label{tab:examples}
\end{center}
}
\end{table}

While the use of $\EBids$ includes an extra complication to the notion of difficulty, it does not add any significant additional cost in its computation, as for $\ESteps$ we need to execute the optimal policy many times. Nonetheless, the most difficult part of the estimation of difficulty is finding the optimal policy $\pi^*$.



The previous sections can be analysed in terms of whether they lead to bounded difficulty functions.
It seems that for the target case (section \ref{sec:PACtarget}) we have that if $r^* = 1$, the difficulty function is unbounded, but otherwise it can be bounded.

} 

\commentAGI{The number of repetitions is related to effort. 
We have argued that this is an upper approximation, but it can be much lower in many occasions. For instance, if an environment gives rewards 0 and 0.6 uniformly randomly independently of the action, so the expected response is 0.3, a threshold on 0.29999999 will lead to a high W but there is nothing to choose, as all policies behave the same, and whatever the agent does would be the same. All other policies get the same result, so there is no need for effort for discarding hypotheses, or dangers in guessing a wrong policy, etc. This is related to unquestionability, as whether there are competing programs of similar complexity is relevant, as in \cite{HernandezOrallo-MinayaCollado1998,HernandezOrallo2000a}. 
However, a Levin search (or a real agent) does not have this information, so all the verification effort has to be done anyway. 
}




\section{Conclusions}\label{sec:conclusions}


As we have mentioned during this paper, the notion of task is common in AI evaluation, in cognition and also in human evaluation. However, a general formalisation, their arrangement and, most especially, their difficulty has not been addressed with earnest determination. \commentAGI{Of course, with this resolution of being general, we have left some other more comfortable approaches, such as MDP and other formalisations in AI. Our main goal was difficulty, as we have seen that this is central to many of the other questions.} Difficulty is seen as computational steps of a Levin search, but this search has to be modified to cover stochastic behaviours. Nonetheless, we have been able to find an expression in terms of the best policy for the task.
These ideas are an evolution and continuation of early notions of task and difficulty in \cite{upsychometrics2} and \cite{orallo2014JAAMAS} respectively.

\commentAGI{
There have been some early approaches where the role of $Kt$ has been explored for different kinds of optimisation or inference problems 
\cite{hernandez2000philosophica,hernandez2000explanatory,HernandezOrallo99a,HernandezOrallo00c,HernandezOrallo00d}. The disposition and arrangement of tasks was discussed in \cite{HernandezOrallo00b}, as well as the notion of task or agent breadth \cite{masum2002turing,goertzel2009toward,rohrer2010accelerating}, and the distinction between specific and general \cite{hernandez2014ai}. 
The notions introduced in this paper, and the expression for difficulty can be useful to reunderstand some of the recent contributions in the evaluation of intelligence \cite{Legg-Hutter2007,Hibbard2009,HernandezOrallo-Dowe2010,HernandezOrallo10b,CAEPIA2011,AGI2011DarwinWallace,AGI2011Compression,AISB-AICAP2012b,manchester2012,AGI2012social,AGI2011Evaluating,AISB-AICAP2012a,IQnotformachines,leggveness2013approximation,DoweHernandez-Orallo2013a_universal,hernandez2013potential,insa2014}.
}

The relevance of verification in difficulty has usually been associated with deduction. However, some works have incorporated it as well in other inference problems, such as induction and optimisation, using Levin's Kt \cite{HernandezOrallo00d,mayfield2007minimal,alpcan2014}. 



\comment{
In AGNT, the set of policies in $\Omega$ is sought using a probability $w$, and we calculate the minimum number $N$ of policies to be sampled from $\Omega$ (without replacement) such that the probability of finding a policy close to the one with maximum reward with tolerance $\gamma$ is at least 0.5.
If $w$ is a normalisation of $2^{-K(\pi)}$, it is shown in proposition 3 in that paper that $N \leq 2^{-K(\pi)}$ where $\pi$ is one of the policies within the tolerance. This suggests the application of a logarithm, which leads to a difficulty function $D$. As $D$ depends on $\gamma$, this is used to create {\em environment response curves} that shows in the same plot $\gamma$ (or reward) and difficulty (or ability $\theta$).
While this result and approach is related to Levin's universal search, this approach (without replacement) is assuming that the policy-searching strategy has memory. Instead, Levin's universal search uses an enumeration approach, which does not require memory (in fact, re-evaluates some policies many times).
}

\comment{
``minimal process for creating $x$ from nothing"
\cite{mayfield2007minimal}
We have the length of the minimal history is $|p'h'| = min \{|ph|: U(p)  x \}$
We have $|ph|2^{|p|}$, where $h$ is the history of them... ``minimal process for creating $x$ from nothing".
$|ph|$ is the time to be executed, which includes $p$ as the length of the program again, because {\em the program has to be loaded}. THIS CAN BE DROPPED IF CONSIDERED PART OF THE EXECUTION STEPS. 
If we calculate the logarithm, we get $\log(|p|) + \log(|h|) + |p|$, i.e., Kt.
}


We can briefly mention some issues that we have not fully developed here. \commentAGI{First, we} \commentARXIV{We} limit difficulty to the complexity of the best policy. However, the notion would be more robust if we considered more policies and their aggregation using a (universal) distribution. \commentAGI{This is in principle possible, but would make the expression more convoluted and the notions of composition and decomposition trickier to analyse. A second issue is that in the second part of the paper we have not discussed the value of $\nuE$ as in $Kt^{[\epsilon, \TAUC  
\mapsto  \nuE]}(\mu)$ in section \ref{sec:KandKt}, because it is said to be given by the Levin's search. However, this could be further investigated.}  
Many other things could be explored, especially around the notions of composition and decomposition, task instance and agent response curves. 
\commentAGI{Also, while our use of `PAC' is just superficially related to PAC learning, we may have a closer look as this, in particular in the context of PAC reinforcement learning \cite{Strehl:2006:PMR:1143844.1143955}.}

\bibliography{biblio}

\end{document}